\documentclass[twoside]{article}

\usepackage[preprint]{aistats2027}

\usepackage[utf8]{inputenc}
\def\withnotes{1}
\def\withcolors{1}
\usepackage{tex_macros}
\usepackage{graphicx}
\usepackage{subcaption}
\usepackage{bm}
\usepackage{framed}
\usepackage{physics}
\usepackage{nicefrac}
\usepackage[round]{natbib}

\PassOptionsToPackage{numbers,compress}{natbib}

\begin{document}

%

%

\twocolumn[

\aistatstitle{Denoising Score Matching in Diffusion Models: The Effect of Rank Constraints on Generalization}

\aistatsauthor{ Anand Jerry George \And  Nicolas Macris }

\aistatsaddress{ École Polytechnique Fédérale de Lausanne (EPFL),\\
Lab for Statistical Mechanics of Inference in Large Systems (SMILS),\\ CH-1015 Lausanne,\\ Switzerland} 
]

\begin{abstract}
Architectural biases restrict the class of functions that a diffusion score network can represent, but their effect on generalization remains poorly understood. We study denoising score matching with score networks having a rank constrained non-linear term and a linear skip connection. Numerical experiments with trainable three-layer networks on data supported on an embedded sphere show that increasing rank lowers training error, but its effect on test error depends on sample size and data dimension. To explain this behavior, we analyze a random-feature score model under rank constraint with Gaussian data supported on a low-dimensional subspace. We characterize the optimal rank-constrained estimator through a truncated singular value decomposition and derive asymptotically precise limiting training and test errors when sample size $(n)$, feature width $(p)$, data dimension $(D)$, and rank $(r)$ scale proportionally with the ambient dimension $(d)$. The resulting spectral characterization distinguishes eigenmodes that improve the empirical fit from those that improve population accuracy. It also identifies the dimension of the orthogonal complement of the data subspace as a natural rank threshold. The theoretical learning curves qualitatively reproduce the observed rank-dependent overfitting, providing a tractable account of how bottlenecks or rank constraints can regularize score estimation and why their effectiveness depends on data dimension and sample size.
\end{abstract}

\section{Introduction}
Diffusion models have achieved remarkable success in high-dimensional generative modeling, yet their ability to memorize training data remains an important theoretical and practical concern \citep{somepalli_diffusion_2023, carlini_extracting_2023}. In the highly overparameterized regime, score networks can in principle fit increasingly fine details of the training dataset, and at sufficiently small noise levels the learned score may become strongly influenced by individual training samples \citep{gu_memorization_2023, yoon_diffusion_2023}. At the same time, practical diffusion architectures often generate novel samples rather than simply reproducing their training sets. Understanding which architectural biases prevent a score model from collapsing toward such memorizing solutions is therefore important for explaining the generalization behavior of diffusion models.

A prominent feature of practical score networks is their encoder--decoder structure \citep{ronneberger_u-net_2015}. Although these models may contain a very large number of parameters, information is repeatedly compressed into lower-dimensional representations before being decoded back to the ambient data space. This suggests that architectural bottlenecks may act as a restriction on the class of score functions that can be learned. In particular, they may prevent the model from independently fitting the score around every training example, while retaining the dominant directions needed to represent the underlying data distribution. This motivates the hypothesis that a form of effective low-rank structure in the learned score can suppress overfitting.

In this work, we isolate this mechanism through two complementary models of the score function. First, we numerically study the learning curves of a three-layer neural network with a bottleneck structure. Second, we consider a random-feature parameterization with an explicit rank constraint on the learned second-layer weight matrix. Both models comprise an additional linear skip connection term. The rank constraint in the random-feature model acts as an idealized encoder–decoder bottleneck, forcing the learned vector field to factor through a low-dimensional latent space. For this model, we derive analytical expressions for the limiting learning curves. By comparing these theoretical predictions with our empirical findings, we investigate how the rank of the learned score affects generalization in a controlled setting.

Our analysis shows how restricting the rank can qualitatively change the learning behavior of denoising score matching. High-rank models have sufficient flexibility to exploit sample-specific directions and can approach solutions strongly tied to the empirical dataset, whereas reducing the rank can suppresses these directions and force the estimator to retain only collective structure shared across samples. The resulting framework provides a simple theoretical model for studying how bottlenecks in practical score architectures can act as an regularizer against memorization/overfitting.

\subsection{Main Findings}
\begin{figure}
    \centering
    \includegraphics[width=0.9\linewidth]{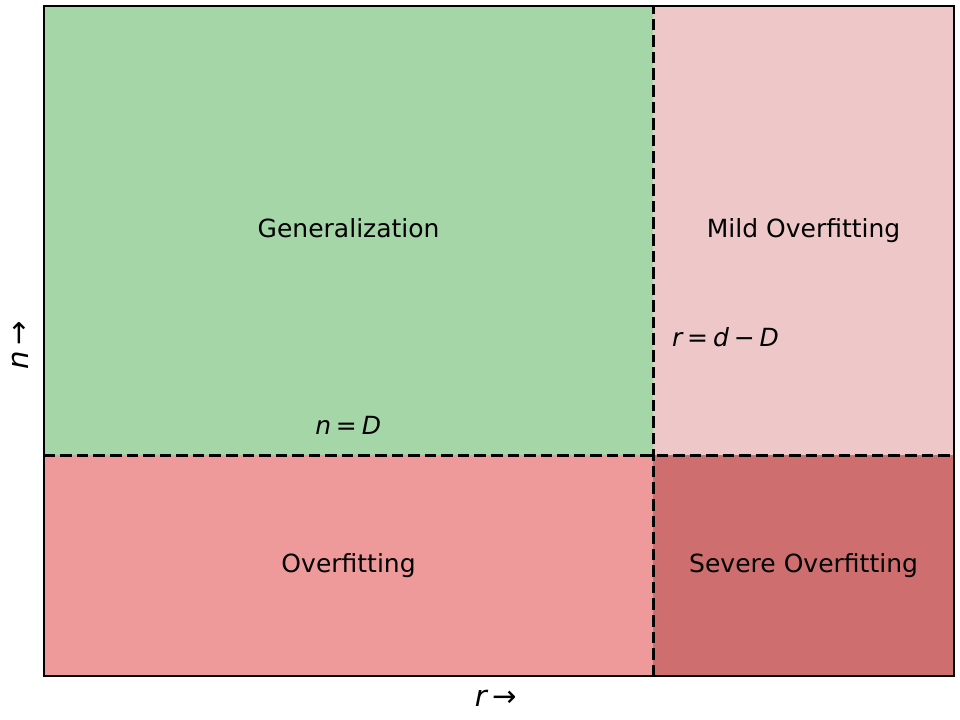}
    \caption{A qualitative illustration of our findings.}
    \label{fig:phase_diagram}
\end{figure}
Our findings reveal how rank of the nonlinear score component, sample size, and data dimension, all affect generalization. This is summarized schematically in Fig.~\ref{fig:phase_diagram}. The horizontal axis represents the rank $r$, while the vertical axis represents the number of training samples $n$. For data spanning a $D$-dimensional subspace of $\mathbb{R}^d$, two reference lines help organize the observed behavior: $n=D$, and $r=d-D$.
\begin{itemize}
    \item When $n>D$, the training samples span the data subspace almost surely in the models considered. Increasing rank below $d-D$ improves both training and test errors by allowing the model to learn more of the score component orthogonal to the data subspace. This behavior motivates us to interpret the upper-left part of Fig.~\ref{fig:phase_diagram} as generalization region.

    \item Increasing rank beyond $d-D$ introduces additional flexibility beyond that required by the data distribution. Training error continues to decrease, while test error can increase or plateau. The upper-right region represents this tendency toward overfitting. Additional training samples can mitigate this effect.

    \item When $n<D$, the training samples leave some directions of genuine data variation unobserved. The model can mistake these directions for normal space to the data subspace, so even a score with rank less than $d-D$ can generalize poorly. This explains the overfitting region below the horizontal boundary and why restricting rank alone is insufficient.

    \item Combining limited sample size with a large rank can further increase the separation between training and test errors, motivating the severe-overfitting region in the lower-right corner. 
\end{itemize}

These regions describe qualitative tendencies rather than sharp phase boundaries. In particular, for certain values of $n/d,D/d$, additional rank can be detrimental even below $d-D$, including when $n>D$. Our random-feature analysis characterizes this dependence through limiting training and test errors, while the numerical study illustrate how it appears in networks with learned features.

\section{Related Works}

\paragraph{Diffusion models and score estimation.}
Diffusion models learn a generative process through denoising and score estimation \citep{sohl-dickstein_deep_2015,song_generative_2019,ho_denoising_2020,song_score-based_2020}. Their sampling accuracy has been studied under assumptions on the error of the learned score \citep{chen_sampling_2022,benton_nearly_2023,chen_improved_2023,bortoli_convergence_2022}. These guarantees motivate a complementary question: how do finite data and restrictions on the score network determine that error? Analyses of score learning and distribution estimation address this question in several tractable settings \citep{cui_analysis_2023,shah_learning_2023,han_neural_2023,chen_score_2023,cui_precise_2025}. Our focus is the effect of an explicit architectural rank constraint on denoising score matching and its generalization error.

\paragraph{Memorization, data geometry, and architectural biases}
Empirical evidence of training-data replication \citep{somepalli_diffusion_2023,carlini_extracting_2023} has motivated theoretical studies of the empirical optimal score and the resulting sampling dynamics \citep{biroli_dynamical_2024,raya_spontaneous_2023,ambrogioni_search_2024,achilli_losing_2024}. Manifold and geometric perspectives further reveal how data structure influences memorization \citep{pidstrigach_score-based_2022,ross_geometric_2024,achilli_memorization_2025, wang_analytical_2025, george_analysis_2025,george_asymptotic_2026, merger_generalization_2026}.
Architectural constraints can favor shared structure over memorization. \citet{kamb_analytic_2024} show that locality and equivariance enable novel recombinations of training-image patches, while \citet{buchanan_edge_2025} identify a capacity-dependent crossover between memorizing and generalizing models' training losses. \citet{ye_provable_2025} distinguish empirical from population score estimation and their network-size requirements, supporting capacity control as a means to reduce memorization. \citet{kadkhodaie_generalization_2023} connect generalization to learned denoisers that perform shrinkage in geometry-adaptive harmonic bases, highlighting how architectural inductive biases shape the directions used for denoising. 

\paragraph{Random features and precise learning curves.}
Random-feature score models provide a tractable setting for studying finite-sample generalization \citep{li_generalization_2023,saha_generalization_2025}. \citet{george_denoising_2026} derive asymptotically precise test and train errors for Gaussian data, identifying the joint influence of sample size, feature width, and the number of noise realizations used in training. \citet{bonnaire_why_2025} identify a separation between the training times at which generalization and memorization emerge. More recently, \citet{latourelle-vigeant_generalization_2026} derive precise risk trajectories for kernel score under gradient flow, distinguishing generalizing, interpolating, and memorizing regimes. Closer to our work, \citet{urfin_first_2026} connect generalization and memorization to distinct spectral components of the neural tangent kernel Gram matrix and investigate spectral truncation as a regularizer. In contrast, we impose explicit output rank constraints on a nonlinear random-feature model for data supported on a low-dimensional subspace.


\section{Preliminaries}
We briefly discuss some basics of diffusion models, and the score models that we use. 
\subsection{Diffusion Models}

Given $n$ i.i.d. samples $x_1,x_2,\cdots,x_n$ from an unknown distribution $P_0$ on $\R^d$, diffusion models draw new samples from $P_0$ by time reversing a stochastic process. This stochastic process gradually transforms $P_0$ into a tractable reference distribution, such as the isotropic Gaussian. We use an Ornstein-Uhlenbeck (OU) process~\citep{gardiner_stochastic_2009} for this forward diffusion:
\begin{equation}\label{eqn:forward_ou}
    \dd X_t = - X_t \;  \dd t +\sqrt{2} \; \dd B_t\;,\quad X_0\sim P_0\;,
\end{equation}
where $B_t$ is a standard Brownian motion in $\R^d$. Conditional on $X_0$, the distribution of $X_t$ is $\cN{a_tX_0,h_tI_d}$, where $a_t=e^{-t}$ and $h_t=1-e^{-2t}$, so that $a_t^2+h_t=1$. As $t\to\infty$, $a_t\to 0$ and $h_t\to 1$, and the distribution of $X_t$ converges to the $d$-dimensional standard Gaussian. For $t>0$, its marginal density is
\begin{equation}\label{eqn:marginal_prob_forward}
    P_t(x) = {(2\pi h_t)^{-d/2}}\int_{\R^d}  \dd P_0(x_0) \; e^{-\frac{\norm{x-a_tx_0}^2}{2h_t}}\;.
\end{equation}
For a fixed $T>0$, the reverse process starts from $Y_T\sim P_T$ and evolves backward in time according to
\begin{equation}\label{eqn:backward_sde}
    -\dd Y_t = \left(Y_t+2\nabla\log P_{t}(Y_t)\right) \;  \dd t+\sqrt{2} \; \dd \tilde{B}_t\;,
\end{equation}
where $\tilde{B}_t$ is a different instance of standard Brownian motion. 
Classical results state that $Y_t$ and $X_t$ have the same marginal distribution at every $t$~\citep{anderson_reverse-time_1982, haussmann_time_1986}, yielding $Y_0\sim P_0$. In practice, the unknown terminal distribution $P_T$ is approximated by $\cN{0,I_d}$ for sufficiently large $T$. Implementing the reverse process also requires estimating the unknown \emph{score function} $\nabla\log P_t$.

Denoising score matching (DSM)~\citep{vincent_connection_2011} estimates this score using noisy observations $y_t=a_tx+\sqrt{h_t}z$, where $x\sim P_0$ and $z\sim\cN{0,I_d}$ are independent. The conditional score provides the tractable target $\nabla\log P_t(y_t\mid x)=-z/\sqrt{h_t}$ for $t>0$, suggesting the population objective
\begin{equation*}
    \cL_{\text{DSM}}(s_t)=
    \frac{1}{d}\shortexpect_{x_0,z}{\norm{\sqrt{h_t}s_t(y_t)+z}^2}.
\end{equation*}
Replacing the expectation over $P_0$ by the training samples gives the empirical loss studied in this work:
\begin{align}\label{eqn:empirical_dsm_loss}
    \cL(s_t) &= \frac{1}{dn}\sum_{i=1}^{n} \shortexpect_z{\norm{\sqrt{h_t}s_t(a_tx_i+\sqrt{h_t}z)+z}^2}.
\end{align}
In practice, the score function $s$ is chosen from a parametric function class and trained by minimizing this empirical objective, subject to the model constraints or regularization. We next introduce the rank-constrained score models used in our study.
\subsection{Score Model for Numerical Study: Rank Constrained 3-Layer NN}

For each diffusion time \(t\), we consider the score model
\begin{equation}\label{eqn:score_numerical}
    s_{t}(y)
    =
    -y
    +
    \rat{B_t}{r} \rat{A_t}{p}
    \act{
        \rat{W_t}{d} y
        +b_t
    },
\end{equation}
where
$W_t\in\mathbb{R}^{p\times d},
\;
B_t\in\mathbb{R}^{d\times r},
\;
A_t\in\mathbb{R}^{r\times p},
\;
b_t\in\mathbb{R}^{p}$ are learnable parameters, and $\varrho$ is a non-linear activation function.
Here, \(p\) is the hidden-layer width and \(r\) controls the rank of
the nonlinear component. The parameters are learned
independently at each \(t\).

The first term, \(-x\), acts as a linear skip connection.
The remaining term provides a nonlinear, learnable correction whose output
belongs to the column space of \(B_t\). Since
$
    \operatorname{rank}(B_t A_t)\leq r,
$
this correction is restricted to an output subspace of dimension at
most \(r\). Small values of \(r\) therefore impose a strong bottleneck, whereas larger values allow the score model
to represent a richer collection of directions.

This parameterization allows us to study whether rank constraining
helps in capturing the population-level geometry of
the data while preventing the model from fitting sample-specific
fluctuations. In particular, we investigate whether increasing \(r\)
reduces the training error while eventually increasing the
train--test gap, indicating a transition from population-level
learning to finite-sample overfitting.

\subsection{Theoretical Score Model: Random Features with a Skip Connection}\label{sec:RFNN}
 For our theoretical analysis, we consider a random-feature neural network \citep{rahimi_random_2007} with fixed first-layer weights, a trainable second layer, and a linear skip connection.
In particular, the score parameterization from $\R^d$ to $\R^d$ given by
 \begin{equation}\label{eqn:RFNN_score}
  s_t(y|W_t) = -y+ A_t\act{W_ty/\sqrt{d}},   
 \end{equation}
where $W_t\in\R^{p\times d}$ has entries drawn i.i.d. from $\cN{0,1}$, $\varrho$ is an activation function applied element-wise, and $A_t\in\R^{d\times p}$ contains the trainable weights.

We impose the constraint $\text{rank}(A_t)\le r$, which restricts the nonlinear correction to an output subspace of dimension at most $r$. This provides a tractable counterpart to the bottleneck architecture used in our numerical study. The feature width $p$ controls the size of the fixed nonlinear representation, while $r$ controls the rank of the learned component.

Random-feature models admit precise theoretical analyses while reproducing phenomena observed in more complex neural networks, including double descent in overparameterized regimes~\citep{mei_generalization_2022,hu_asymptotics_2024,bodin_model_2021,bodin_gradient_2022}.
\section{Numerical Study}
For our numerical experiments, the data distribution $P_0$ is uniform distribution on a sphere embedded in $\mathbb{R}^{d}$:
$x=\Theta\xi,\;
\Theta^\top \Theta=I_D,\;
\xi=\sqrt{D}\,\frac{g}{\|g\|},\;
g\sim\mathcal{N}(0,I_D).
$
Thus the data manifold has intrinsic dimension
$D-1$, and is embedded in a linear subspace of dimension $D$.

For each rank $r$, we train a separate score model given in \eqref{eqn:score_numerical}.
We sweep $r$ and
optimize the denoising score-matching objective \eqref{eqn:empirical_dsm_loss}.
Fig.~\ref{fig:sim_model_tanh_t0p01} reports train and test error estimates of this
same objective after $30$k Adam steps, averaged over three random seeds. More details on the simulation can be found in Appendix~\ref{app:simulation_conditions}.

\begin{figure*}[htbp]
    \centering
    \includegraphics[width=\linewidth]{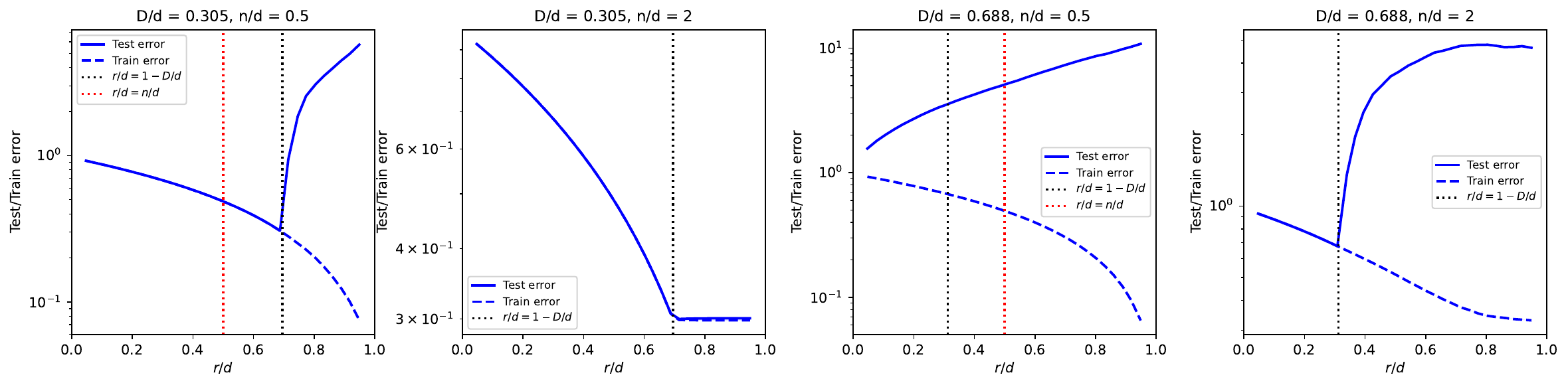}
    \caption{Simulation results for our numerical model, with $\tanh$ activation at $t=0.01$, $p/d=2$.}
    \label{fig:sim_model_tanh_t0p01}
\end{figure*}

First, we observe that the train error decreases monotonically with increasing $r$ across all settings. However, the behavior of test error depends on the samples size $n$, and data dimension $D$.
For $n>D$, the test error decreases as we increase $r$ till $d-D$. Beyond $d-D$, we observe an abrupt change in behavior: with additional rank, the test error increases or plateaus, indicating finite-sample overfitting. Increasing $n$ suppresses this effect strongly at smaller $D$,
but does not eliminate it at larger $D$. When $n<D$, the sudden change in test error at $r = d-D$ disappears. This is because the clean training samples cannot span the full data
subspace, making the unseen data subspace directions a source of poor generalization. 

The observed rank dependence is consistent with learning the score component
in the space orthogonal to the data span, whose dimension is $d-D$. Section~\ref{sec:discussion} discusses these aspects in detail.
\section{Theoretical Study}
We analyze the random-feature score model given in \eqref{eqn:RFNN_score} at a fixed diffusion time. Suppressing the time index, we write
\[
    s_A(y)=-y+A\phi(y),\qquad
    \phi(y)=\act{Wy/\sqrt{d}},
\]
where $W$ is held fixed and $A$ is learned subject to a rank constraint. We consider Gaussian data $x\sim\mathcal{N}(0,\cC)$ with $\cC=\operatorname{diag}(I_D,0)$, providing a tractable model of data supported on a $D$-dimensional subspace of the ambient space. 

Given independent training samples $x_1,\ldots,x_n$, let $y_i=ax_i+\sqrt{h}z$, where $z\sim\mathcal{N}(0,I_d)$ is independent of the data. The rank-constrained DSM estimator solves
\begin{equation}\label{eqn:rank_constrained_min_pblm}
   \min_{\text{rank}(A)\le r} \cL(A) = \sum_{i=1}^n\shortexpect_z{\norm{\sqrt{h}(A\phi(y_i)-y_i)+z}^2}.
\end{equation}

To characterize the solution, define matrices
\begin{align*}
    U&=\frac{1}{n}\sum_{i=1}^n\bE{z}{\phi(y_i)\phi(y_i)^T},\\
    V&=\frac{1}{n}\sum_{i=1}^n
    \bE{z}{\phi(y_i)(y_i-z/\sqrt{h})^T}.
\end{align*}
Here, $U$ and $V$ are the empirical feature second moment and feature–target cross moment, respectively. The following theorem gives an exact finite-dimensional solution to \eqref{eqn:rank_constrained_min_pblm} through a truncated singular value decomposition.

\begin{theorem}[Rank-constrained Random Features estimator]
\label{thm:rank_constrained_dsm}
Fix $t>0$, let $a=e^{-t}$ and $h=1-e^{-2t}$, and condition on the training samples and $W$. Assume that $U\succ0$. Let $q=\min(d,p)$ and $r\in\{0,\ldots,q\}$. Write $V^TU^{-1/2}=S\Sigma R^T,$ with singular values $s_1\geq\cdots\geq s_q\geq0$, and let $\Sigma_r$ retain the $r$ largest singular values. Then a solution to \eqref{eqn:rank_constrained_min_pblm} is
\begin{equation}\label{eq:main_rank_optimizer}
    \widehat A_r=S\Sigma_rR^TU^{-1/2}.
\end{equation}
\end{theorem}

\begin{proof}[Proof Sketch]
The normal equation $A_{\mathrm{LS}}U=V^T$ gives
\begin{align*}
    \cL_{\mathrm{train}}(A)
    &=\cL_{\mathrm{train}}(A_{\mathrm{LS}})+\frac{h}{d}\|(A-A_{\mathrm{LS}})U^{1/2}\|_F^2.
\end{align*}
Since multiplication by $U^{1/2}$ preserves rank, the Eckart--Young theorem yields~\eqref{eq:main_rank_optimizer}.
\end{proof}
Equation~\ref{eq:main_rank_optimizer} shows that rank truncation retains the leading singular modes after accounting for the empirical feature covariance. 

We next state the limiting characterization of test and train errors for the rank constrained solution. First, we define test and train loss as follows:
\begin{align*}
\cL_{\text{test}}(A) &= \frac{1}{d} \shortexpect_{x_0,z}{\norm{\sqrt{h_t}s_A(a_tx_0+\sqrt{h_t}z)+z}^2},\\
    \cL_{\text{train}}(A) &= \frac{1}{dn}\sum_{i=1}^{n} \shortexpect_z{\norm{\sqrt{h_t}s_A(a_tx_i+\sqrt{h_t}z)+z}^2}.
\end{align*}
\begin{assumption}[Activation function]\label{assmptn:activation}
    We assume that for $g\sim\mathcal N(0,1)$, and for all $\kappa>0$, $\bE{g}{\varrho(\kappa g)}=0$. Set $\sigma^2=a^2\psi_D+h$, and $\mu=\frac{1}{\sigma}\bE{g}{g\varrho(\sigma g)}$. Constants $s, v,$ and $v_0$ are as defined in eqns \eqref{eqn:s2}, \eqref{eqn:v2}, \eqref{eqn:v02} in Appendix~\ref{app:rank_constrained_learning}. We assume that $\mu$ and $s$ are positive.
\end{assumption}
\begin{theorem}[Limiting characterization of test and train errors]
\label{thm:limiting_test_error}
Fix $t>0$ and let $d\to\infty$ with $\psi_\diamond = \frac{\diamond}{d}$, for $\diamond\in\{n,p,D,r\}$ held fixed. 
Let the training samples be independent draws from $\mathcal N(0,\operatorname{diag}(I_D,0))$, independent of the i.i.d. standard Gaussian entries of $W$. Let $\varrho$ satisfy Assumption~\ref{assmptn:activation}, and suppose that the Gaussian equivalence principle given in Assumption~\ref{asmptn:GEP} of Appendix~\ref{app:rank_constrained_learning}. is valid. 

Let $G_{00}, G_{17}, G_{22}, G_{44}$ be obtained by solving the system of fixed point equations~\eqref{eqn:fixed_point_system} corresponding to the limiting block resolvent of~\eqref{eqn:linear_pencil}. For $\Im z>0$, define
\begin{align}
    e_1(z)&=-\frac{\psi_p}{s^2z}G_{00}(z),\quad
    e_2(z)=-\frac{\psi_D}{a^2\mu}G_{17}(z),\notag\\
    e_3(z)&=\frac{1-\psi_D}{a^4\mu^2}\bigl(1-G_{22}(z)\bigr).
    \label{eq:main_limiting_resolvents}
\end{align}
The limiting eigenvalue distribution $\nu$ of $M=V^TU^{-1}V$ has Stieltjes transform
\begin{equation}\label{eq:main_limiting_stieltjes}
    \textbf{s}_{\nu}(z)=\frac{\psi_D(1-G_{44}(z))
    +(1-\psi_D)(1-G_{22}(z))-1}{z}.
\end{equation}
Suppose there is a $\Gamma_{\psi_r}>0$ satisfying
\begin{equation}\label{eq:main_rank_cutoff}
    \nu((\Gamma_{\psi_r},\infty))=\psi_r.
\end{equation}
Define
\begin{equation}\label{eq:main_limiting_risk_transform}
    \textbf{m}(z)=\mu^2e_2(z)
    +\left(h\mu^2-\frac{2a^4\mu^2}{z}\right)e_3(z)
    +v^2e_1(z).
\end{equation}
Then, with $\omega = 1+h(\sigma^2-2)$, the estimator~\eqref{eq:main_rank_optimizer} satisfies
\begin{align*}
    \cL_{\mathrm{test}}(\widehat A_r)
    &\xrightarrow{\mathbb P}\omega
    +\frac{h}{\pi}\lim_{K\to\infty}\lim_{\eta\downarrow0}
    \int_{\Gamma_{\psi_r}}^{K}
    x\,\Im \textbf{m}(x+i\eta)\,\dd x,\\
    \cL_{\mathrm{train}}(\widehat A_r)
    &\xrightarrow{\mathbb P}\omega-\frac{h}{\pi}\lim_{K\to\infty}\lim_{\eta\downarrow0}
    \int_{\Gamma_{\psi_r}}^{K}
    x\,\Im \textbf{s}_{\nu}(x+i\eta)\,\dd x.
\end{align*}
\end{theorem}

\begin{proof}[Proof sketch]
We outline the proof here. Complete steps can be found in Appendix~\ref{app:rank_constrained_learning}. Expanding the DSM losses gives quadratic, cross, and  $\widehat A_r$ independent terms; see~\eqref{eq:app_test_expansion} and~\eqref{eq:app_train_expansion}. The test error involves population moments, while the training error involves the empirical moments that define the optimizer.

The first step is to obtain tractable representations of these moment matrices. Using Hermite expansions and truncation up to leading order give~\eqref{eq:app_moment_representations}. The \textit{Gaussian equivalence principle} \citep{gerace_generalisation_2020, goldt_gaussian_2022, hu_universality_2023,montanari_universality_2022} then replaces the nonlinear noise-averaged feature matrix $H=\varrho_0(WX/\sqrt d)$, for the limiting trace quantities under consideration, by
\[
    H^G=a\mu WX/\sqrt{d}+v_0\Omega,
\]
where $\Omega$ has independent standard Gaussian entries and is independent of $W$ and $X$. The linear term preserves the feature--data correlations, while $v_0\Omega$ represents the residual feature variance. This is an asymptotic equivalence for the relevant observables, rather than an entrywise identity. It reduces the moment matrices to expressions involving Gaussian random matrices and the coefficients $\mu,s,v,v_0$.

Next, Theorem~\ref{thm:rank_constrained_dsm} implements the rank constraint by retaining the leading singular modes of $V^TU^{-1/2}$. Their squared singular values are the nonzero eigenvalues shared by $M=V^TU^{-1}V$ and $\tilde M=U^{-1/2}VV^TU^{-1/2}$. Stone's formula converts this truncation into a resolvent integral. For example, with a cutoff $\Gamma$ separating the retained eigenvalues and an upper endpoint $K$ above the spectrum,
\[
    S\Sigma_r\Sigma_r^TS^T
    =\frac{1}{\pi}\lim_{\eta\downarrow0}
    \int_{\Gamma}^{K}x\,\Im(M-(x+i\eta)I)^{-1}\,\dd x.
\]
The analogous formula for $\tilde M$ handles the quadratic term. In the proportional limit, the rank fraction selects the upper-tail cutoff $\Gamma_{\psi_r}$ through~\eqref{eq:main_rank_cutoff}.

These representations reduce the losses to normalized traces of rational expressions in random matrices. Set $Q(z)=(VV^T-zU)^{-1}$. The identity
\[
    (M-zI)^{-1}
    =z^{-1}\bigl(V^TQ(z)V-I\bigr)
\]
expresses the Stieltjes transform of $M$ through $E_0(z)=d^{-1}\tr{V^TQ(z)V}$, while the population weights produce the traces $E_1,E_2,E_3$ in~\eqref{eq:app_weighted_traces}. To compute their limits, the \textit{linear-pencil} technique embeds these rational expressions into selected inverse blocks of the larger matrix $L(z)$ in~\eqref{eqn:linear_pencil}. Its random blocks depend linearly on the underlying Gaussian matrices, and block matrix inverse recovers the required resolvents. A computation rooted in operator-valued free probability theory \citep{far_spectra_2006, helton_applications_2018, adlam_neural_2020, bodin_random_2024,subramonian_auto-fpt_2025} yields the fixed-point system~\eqref{eqn:fixed_point_system} for the limiting normalized block traces. Extracting the appropriate blocks gives~\eqref{eq:main_limiting_resolvents} and~\eqref{eq:main_limiting_stieltjes}.

Finally, summing up the quadratic ($\mu^2e_2+h\mu^2e_3+v^2e_1$) and cross ($-2a^4\mu^2e_3/z$) contributions to test error gives~\eqref{eq:main_limiting_risk_transform}. For training error, the empirical normal equation combines the quadratic and cross terms into a single term. Applying spectral integration with the rank cutoff and adding the constant term $1+h(\sigma^2-2)$ gives the stated learning curves. The spectral integrals are understood with a finite upper endpoint beyond the spectral support, before taking $\eta\downarrow0$.
\end{proof}

The characterization reduces the computation into solving the fixed-point equations \eqref{eqn:fixed_point_system}, finding the spectral cutoff~\eqref{eq:main_rank_cutoff}, and evaluating two integrals. Rank enters through the cutoff $\Gamma_{\psi_r}$, while $\psi_n$, $\psi_p$, and $\psi_D$ shape the resolvent functions $\textbf{m}$ and $\textbf{s}_\nu$.
\begin{figure*}[htbp]
    \centering
    \includegraphics[width=\linewidth]{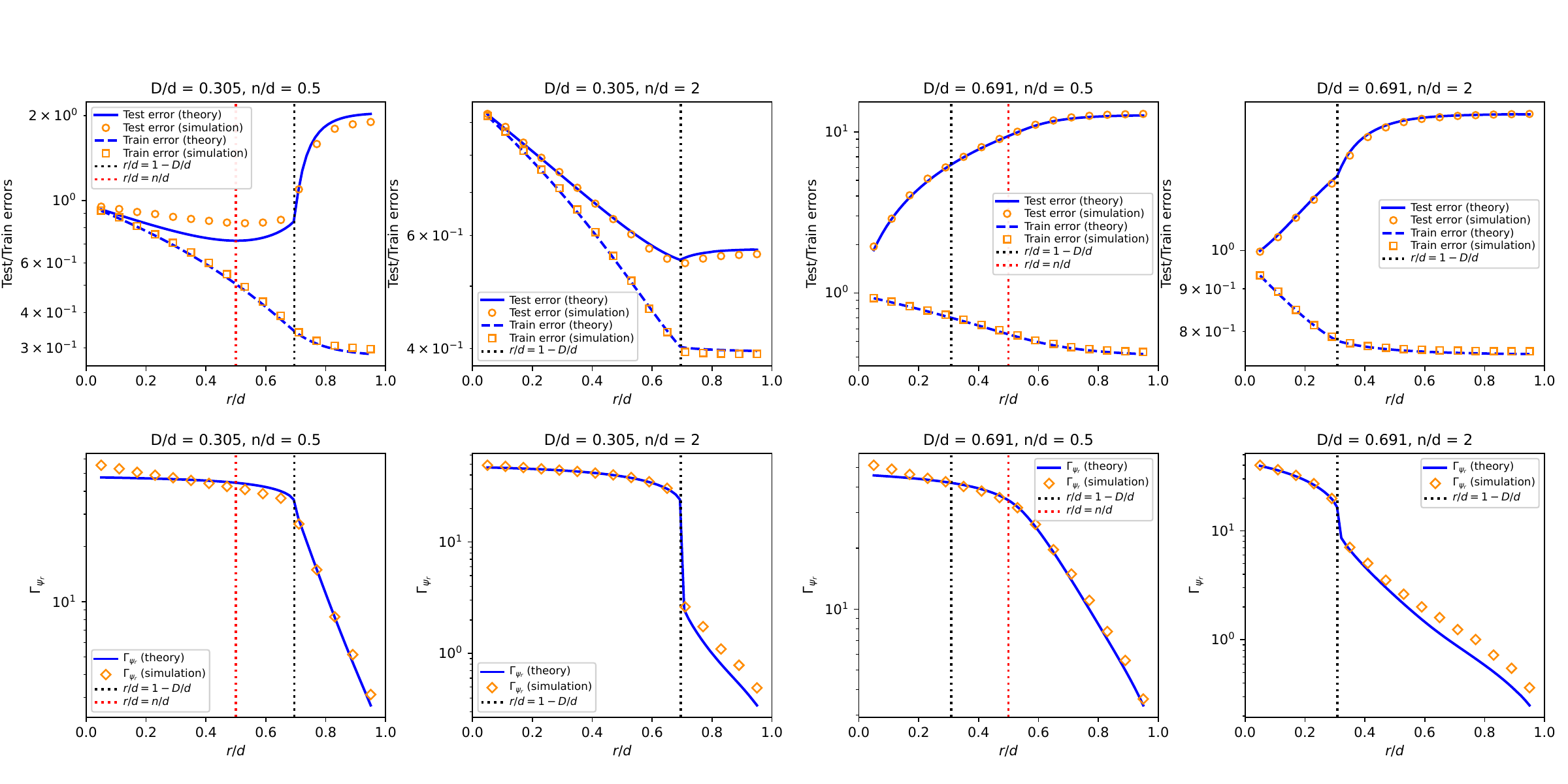}
    \caption{Analytical and finite dimensional $(d=200)$ simulation results for our theoretical model with $\tanh$ activation at $t=0.01$, $\psi_p = 2.0$. Top line: test and train errors. Bottom line: Spectral selection thresholds $\Gamma_{\psi_r}$.}
    \label{fig:theory_model_tanh_t0p01}
\end{figure*}
Fig.~\ref{fig:theory_model_tanh_t0p01} presents the theoretical learning curves and finite-dimensional simulations, which qualitatively reproduce the numerical study's dependence on rank, sample size, and data dimension. Training error decreases with rank, whereas test error exhibits distinct behaviors. For smaller $\psi_D$ and $\psi_n$, test error initially decreases before rising sharply at $\psi_r=1-\psi_D$, while increasing $\psi_n$ substantially suppresses this deterioration. At larger $\psi_D$ and smaller $\psi_n$, test error increases throughout the rank range, showing that additional rank can be detrimental even below $\psi_r=1-\psi_D$. For large $\psi_n$ however, the variation in test error is small. Section~\ref{sec:discussion} discusses the mechanisms underlying these observations. Fig.~\ref{fig:spectrum_theory_model_tanh_t0p01} shows the limiting spectrum of $M$, whose leading modes are retained by the rank constraint. 
\begin{figure*}[t]
    \centering
    \includegraphics[width=\linewidth]{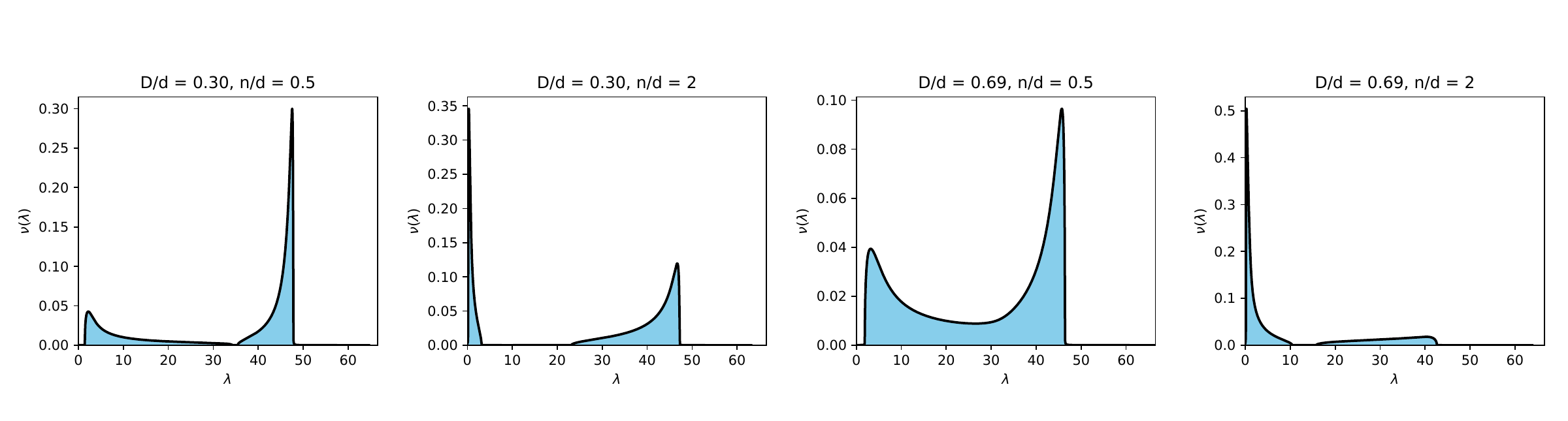}
    \caption{Spectrum of $M$ with $\tanh$ activation at $t=0.01$, $\psi_p = 2.0$.}
    \label{fig:spectrum_theory_model_tanh_t0p01}
\end{figure*}


\section{Discussion}\label{sec:discussion}

Figs.~\ref{fig:sim_model_tanh_t0p01} and~\ref{fig:theory_model_tanh_t0p01} (and Figs.~\ref{fig:sim_model_tanh_t0p1} and \ref{fig:theory_model_tanh_t0p1} in Appendix~\ref{sec:add_res_time}) show that increasing the rank of the nonlinear score component improves the fit to the training dataset without always improving generalization. We organize these observations around the four regimes in Fig.~\ref{fig:phase_diagram}, determined by the sample threshold $n=D$ and the reference rank $r=d-D$. These thresholds help explain the competing effects of model expressivity and estimation errors, but do not define sharp boundaries between decreasing and increasing test error.

\paragraph{Sufficient samples and restricted rank: $n>D$, $r<d-D$.}
When $n>D$, the clean training samples span the data subspace almost surely for the distributions considered here. In this regime, increasing rank can improve generalization by allowing the nonlinear score component to represent more directions orthogonal to that subspace. The underlying geometry is particularly transparent in the Gaussian case: the noisy covariance is $\tilde{\cC}=\text{diag}(I_D,hI_{d-D})$, and the ideal score is $\nabla\log P_t(y)=-y-\frac{a^2}{h}\text{diag}(0,I_{d-D})y$.
The skip connection already gives the correct score along the data subspace, while the remaining component acts in its normal space and has rank $d-D$. A rank below $d-D$ cannot represent this component fully, and the factor $a^2/h$ makes the missing normal directions especially consequential at small diffusion time. This mechanism explains the initial decrease of both training and test errors observed at smaller $D/d$. The spectra in Fig.~\ref{fig:spectrum_theory_model_tanh_t0p01} offer a complementary perspective: at $D/d=0.30$, an upper band is separated from a lower component, with clearer separation for larger $n/d$. In the population setting with an unconstrained linear score, the analogue of matrix $M$ is given by $(I-\tilde{\cC}^{-1})^2\tilde{\cC}$. It has $d-D$ nonzero eigenvalues, all equal to $a^4/h$. This reference spectrum suggests that the upper band in Fig.~\ref{fig:spectrum_theory_model_tanh_t0p01} primarily reflects normal-space directions. However, it should be noted that the empirical setting and the random-feature representation alter both the spectrum and the retained directions. Consequently, even when $n>D$, additional rank below $d-D$ need not always improve test error.

\paragraph{Sufficient samples and excess rank: $n>D$, $r>d-D$.}
Beyond $d-D$, additional rank is no longer needed to represent the normal-space component of the Gaussian population score. It nevertheless provides further flexibility to fit the training samples. Theorem~\ref{thm:rank_constrained_dsm} makes this distinction precise: the rank constraint retains the largest singular modes of $V^TU^{-1/2}$, or equivalently the largest eigenvalues of $M$, and admitting the next mode decreases the training loss by $hs_{r+1}^2/d$. The effect on test error also depends on the direction of that mode and the population feature moments. In Theorem~\ref{thm:limiting_test_error}, this distinction appears through function $\textbf{m}(z)$ which is different from the pure spectral object $\textbf{s}_\nu$. Thus, a large empirical eigenvalue establishes a training benefit without ensuring a corresponding test benefit. Additional modes can capture sample-specific structure, leading to a plateau or an increase in test error despite continued improvement in training error. This interpretation is consistent with the approximate plateau observed for the learned network at smaller $D/d$ and $n/d=2$, and with the more pronounced train--test separation at larger $D/d$. Increasing sample size reduces estimation uncertainty, but does not eliminate the effects of excessive rank at the sample sizes considered. 

\paragraph{Insufficient samples and restricted rank: $n<D$, $r<d-D$.}
Restricting rank is less effective when the training samples do not span the true data subspace. When $n<D$, at least $D-n$ directions of genuine data variation remain unobserved. A model can confuse these directions with the normal space and erroneously learn a strong denoising response along these directions. 
Due to its larger magnitude, the learned score component along the unseen directions are retained even under rank constraint.
 This mechanism is consistent with the increasing test error at larger $D/d$ and $n/d=0.5$, even within the range $r<d-D$. At $D/d=0.69$ and $n/d=0.5$, Fig.~\ref{fig:spectrum_theory_model_tanh_t0p01} shows appreciable density between the lower and upper parts of the spectrum, consistent with a less distinct separation of spectral components. The spectrum alone does not identify the alignment of individual eigenmodes, but it illustrates why selecting eigenmodes that benefit the empirical training need not isolate the directions required for generalization.

\paragraph{Insufficient samples and excess rank: $n<D$, $r>d-D$.}
Combining insufficient samples with a large rank allows both sources of overfitting to act together. The training data leave part of the true data subspace unresolved, while the additional rank permits the model to fit more directions beyond those needed for the population normal-space component. Training error continues to decrease, but this added flexibility can strengthen responses along unobserved data directions and capture further sample-specific structure, widening the train--test gap. This combination motivates the severe-overfitting region in Fig.~\ref{fig:phase_diagram} and is consistent with the substantial deterioration at higher ranks for larger $D/d$ and $n/d=0.5$. 

 Fig.~\ref{fig:theory_model_vs_psin_tanh_t0p01}, which plots training and test errors against $\psi_n$ for two values of $\psi_r$, further illustrates how sample size changes the consequences of a given rank constraint. Once $n>D$, the ambiguity caused by an unobserved part of the data span essentially disappears.

\begin{figure}[htbp]
    \centering
    \includegraphics[width=\linewidth]{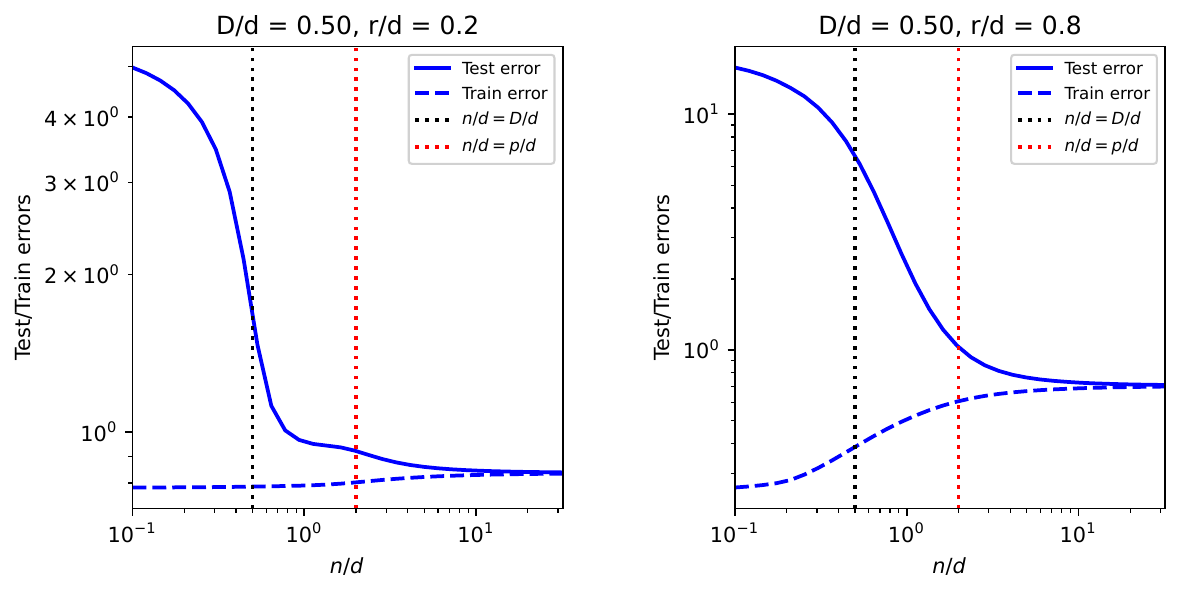}
    \caption{Theoretical results with $\tanh$ activation at $t=0.01$, $\psi_p = 2.0$.}
    \label{fig:theory_model_vs_psin_tanh_t0p01}
\end{figure}

Together, the numerical and theoretical results show that a useful rank constraint depends on both the data dimension and the sample size available to estimate the data geometry. The random-feature model isolates a mechanism for rank-dependent overfitting: empirical spectral selection rewards training fit, whereas generalization also depends on the population alignment of the selected directions. 

\section{Conclusion, Limitations, and Future Work}
We studied how output-rank constraints affect generalization in denoising score matching. Experiments with trainable networks and a proportional-limit random-feature analysis show that additional rank can improve training fit while worsening test error. Conversely, a rank-constraint can indeed suppress overfitting. The spectral characterization of $M$ explains this tradeoff through the population contributions of the retained modes, and highlights the joint role of data geometry and sample size. These findings establish rank constraints as a useful, but data-dependent, form of regularization for score estimation.

Our study isolates the effect of rank, which requires a tractable framework for rank control. The theoretical analysis considers fixed random features and Gaussian data supported on a linear subspace, while the numerical experiments use trainable networks with data on an embedded sphere. Their qualitative agreement supports the proposed mechanism, although the theory is not intended to predict every learning curve of a network with learned features. Moreover, the errors are evaluated at a fixed diffusion time $t$. Connecting these results to generated-sample quality and training-data replication requires studying the complete reverse process. Explicit output-rank constraints thus provide a tractable model of an architectural bottleneck, whose relationship to practical encoder--decoder networks remains to be explored.

Future work includes extending the analysis to learned features and more general data geometries, and developing frameworks to study low-rank effects in practical networks.

\subsubsection*{Acknowledgements}
The work of A. J. G. has been supported by Swiss National Science Foundation grant number 200021-204119. 
\subsection*{AI use statement}

In this work, we have not used generative AI tools for problem formulation, mathematical analysis, and the rest of the required disclosure tasks are not applicable to this work.
We used generative AI tools for 1) developing codes required for simulations and plotting, and 2) assist in writing the paper.
We have reviewed all AI-assisted work. LLM-generated code was verified and tested for correctness by an author. Portions of text written with the help of AI tools were reviewed carefully by the authors.
We take responsibility for the final content of this work, including text, claims, or artifacts produced with the aid of generative AI.


\bibliography{references} 

@misc{urfin_first_2026,
	title = {First {Learn}, {Then} {Memorize}: {The} {Spectral} {Bias} of {Diffusion} {Models}},
	shorttitle = {First {Learn}, {Then} {Memorize}},
	url = {http://arxiv.org/abs/2609.35377},
	doi = {10.48550/arXiv.2609.35377},
	urldate = {2026-10-06},
	publisher = {arXiv},
	author = {Urfin, Raphaël and Bonnaire, Tony and Biroli, Giulio and Mézard, Marc},
	month = sep,
	year = {2026},
	note = {arXiv:2609.35377 [cs.LG]},
}

@misc{hu_asymptotics_2024,
	title = {Asymptotics of {Random} {Feature} {Regression} {Beyond} the {Linear} {Scaling} {Regime}},
	url = {http://arxiv.org/abs/2403.08160},
	doi = {10.48550/arXiv.2403.08160},
	language = {en},
	urldate = {2024-12-02},
	publisher = {arXiv},
	author = {Hu, Hong and Lu, Yue M. and Misiakiewicz, Theodor},
	month = mar,
	year = {2024},
	note = {arXiv:2403.08160 [stat]},
}

@inproceedings{montanari_universality_2022,
	title = {Universality of empirical risk minimization},
	issn = {2640-3498},
	url = {https://proceedings.mlr.press/v178/montanari22a.html},
	language = {en},
	urldate = {2026-10-05},
	booktitle = {Proceedings of {Thirty} {Fifth} {Conference} on {Learning} {Theory}},
	publisher = {PMLR},
	author = {Montanari, Andrea and Saeed, Basil N.},
	month = jun,
	year = {2022},
	pages = {4310--4312},
}

@inproceedings{ye_provable_2025,
	title = {Provable {Separations} between {Memorization} and {Generalization} in {Diffusion} {Models}},
	url = {https://openreview.net/forum?id=42gfTZzyvV},
	language = {en},
	urldate = {2026-10-01},
	author = {Ye, Zeqi and Zhu, Qijie and Tao, Molei and Chen, Minshuo},
	month = oct,
	year = {2025},
}

@misc{subramonian_auto-fpt_2025,
	title = {auto-fpt: {Automating} {Free} {Probability} {Theory} {Calculations} for {Machine} {Learning} {Theory}},
	shorttitle = {auto-fpt},
	url = {http://arxiv.org/abs/2504.10754},
	doi = {10.48550/arXiv.2504.10754},
	urldate = {2026-10-01},
	publisher = {arXiv},
	author = {Subramonian, Arjun and Dohmatob, Elvis},
	month = apr,
	year = {2025},
	note = {arXiv:2504.10754 [cs.LG]},
}

@misc{george_asymptotic_2026,
	title = {Asymptotic {Learning} {Curves} for {Diffusion} {Models} with {Random} {Features} {Score} and {Manifold} {Data}},
	url = {http://arxiv.org/abs/2603.22962},
	doi = {10.48550/arXiv.2603.22962},
	urldate = {2026-10-01},
	publisher = {arXiv},
	author = {George, Anand Jerry and Macris, Nicolas},
	month = jun,
	year = {2026},
	note = {arXiv:2603.22962 [cs.LG]},
}

@misc{latourelle-vigeant_generalization_2026,
	title = {Generalization, memorization, and overfitting for diffusion models trained in the lazy high-dimensional regime},
	url = {http://arxiv.org/abs/2608.23938},
	doi = {10.48550/arXiv.2608.23938},
	urldate = {2026-09-27},
	publisher = {arXiv},
	author = {Latourelle-Vigeant, Hugo and Chewi, Sinho and Pooladian, Aram-Alexandre and Sous, John and Misiakiewicz, Theodor},
	month = aug,
	year = {2026},
	note = {arXiv:2608.23938 [stat.ML]},
}

@article{george_denoising_2026,
	title = {Denoising score matching with random features: insights on diffusion models from precise learning curves*},
	volume = {2026},
	issn = {1742-5468},
	shorttitle = {Denoising score matching with random features},
	url = {https://doi.org/10.1088/1742-5468/ae9379},
	doi = {10.1088/1742-5468/ae9379},
	language = {en},
	number = {8},
	urldate = {2026-09-02},
	journal = {Journal of Statistical Mechanics: Theory and Experiment},
	publisher = {IOP Publishing},
	author = {George, Anand Jerry and Veiga, Rodrigo and Macris, Nicolas},
	month = aug,
	year = {2026},
	pages = {084009},
}

@misc{merger_generalization_2026,
	title = {Generalization {Dynamics} of {Linear} {Diffusion} {Models}},
	url = {http://arxiv.org/abs/2505.24769},
	doi = {10.48550/arXiv.2505.24769},
	urldate = {2026-09-02},
	publisher = {arXiv},
	author = {Merger, Claudia and Goldt, Sebastian},
	month = jan,
	year = {2026},
	note = {arXiv:2505.24769 [stat.ML]},
}

@inproceedings{wang_analytical_2025,
	title = {An {Analytical} {Theory} of {Spectral} {Bias} in the {Learning} {Dynamics} of {Diffusion} {Models}},
	url = {https://openreview.net/forum?id=SDhOClkyqC},
	language = {en},
	urldate = {2026-09-02},
	author = {Wang, Binxu and Pehlevan, Cengiz},
	month = oct,
	year = {2025},
}

@misc{bodin_gradient_2022,
	title = {Gradient flow in the gaussian covariate model: exact solution of learning curves and multiple descent structures},
	shorttitle = {Gradient flow in the gaussian covariate model},
	url = {http://arxiv.org/abs/2212.06757},
	doi = {10.48550/arXiv.2212.06757},
	urldate = {2026-02-04},
	publisher = {arXiv},
	author = {Bodin, Antoine and Macris, Nicolas},
	month = dec,
	year = {2022},
	note = {arXiv:2212.06757 [stat]},
}

@inproceedings{george_analysis_2025,
	title = {Analysis of {Diffusion} {Models} for {Manifold} {Data}},
	issn = {2157-8117},
	doi = {10.1109/ISIT63088.2025.11195641},
	urldate = {2026-02-04},
	booktitle = {2025 {IEEE} {International} {Symposium} on {Information} {Theory} ({ISIT})},
	author = {George, Anand Jerry and Veiga, Rodrigo and Macris, Nicolas},
	month = jun,
	year = {2025},
	note = {ISSN: 2157-8117},
	pages = {1--6},
}

@article{mei_generalization_2022,
	title = {The {Generalization} {Error} of {Random} {Features} {Regression}: {Precise} {Asymptotics} and the {Double} {Descent} {Curve}},
	volume = {75},
	shorttitle = {The {Generalization} {Error} of {Random} {Features} {Regression}},
	language = {en},
	number = {4},
	urldate = {2025-01-13},
	journal = {Communications on Pure and Applied Mathematics},
	author = {Mei, Song and Montanari, Andrea},
	year = {2022},
	pages = {667--766},
}

@misc{buchanan_edge_2025,
	title = {On the {Edge} of {Memorization} in {Diffusion} {Models}},
	url = {http://arxiv.org/abs/2508.17689},
	doi = {10.48550/arXiv.2508.17689},
	urldate = {2025-08-29},
	publisher = {arXiv},
	author = {Buchanan, Sam and Pai, Druv and Ma, Yi and Bortoli, Valentin De},
	month = aug,
	year = {2025},
	note = {arXiv:2508.17689 [cs]},
}

@misc{saha_generalization_2025,
	title = {Generalization {Bound} for {Diffusion} {Models} using {Random} {Features}},
	url = {http://arxiv.org/abs/2310.04417},
	doi = {10.48550/arXiv.2310.04417},
	urldate = {2025-09-28},
	publisher = {arXiv},
	author = {Saha, Esha and Tran, Giang},
	month = aug,
	year = {2025},
	note = {arXiv:2310.04417 [stat]},
}

@misc{bonnaire_why_2025,
	title = {Why {Diffusion} {Models} {Don}'t {Memorize}: {The} {Role} of {Implicit} {Dynamical} {Regularization} in {Training}},
	shorttitle = {Why {Diffusion} {Models} {Don}'t {Memorize}},
	url = {http://arxiv.org/abs/2505.17638},
	doi = {10.48550/arXiv.2505.17638},
	urldate = {2025-09-28},
	publisher = {arXiv},
	author = {Bonnaire, Tony and Urfin, Raphaël and Biroli, Giulio and Mézard, Marc},
	month = may,
	year = {2025},
	note = {arXiv:2505.17638 [cs]},
}

@inproceedings{ross_geometric_2024,
	title = {A {Geometric} {Framework} for {Understanding} {Memorization} in {Generative} {Models}},
	language = {en},
	urldate = {2025-05-15},
	booktitle = {The {Thirteenth} {International} {Conference} on {Learning} {Representations}},
	author = {Ross, Brendan Leigh and Kamkari, Hamidreza and Wu, Tongzi and Hosseinzadeh, Rasa and Liu, Zhaoyan and Stein, George and Cresswell, Jesse C. and Loaiza-Ganem, Gabriel},
	month = oct,
	year = {2024},
}

@misc{kingma_adam_2017,
	title = {Adam: {A} {Method} for {Stochastic} {Optimization}},
	shorttitle = {Adam},
	doi = {10.48550/arXiv.1412.6980},
	urldate = {2025-04-16},
	publisher = {arXiv},
	author = {Kingma, Diederik P. and Ba, Jimmy},
	month = jan,
	year = {2017},
	note = {arXiv:1412.6980 [cs]},
}

@article{li_generalization_2023,
	title = {On the {Generalization} {Properties} of {Diffusion} {Models}},
	volume = {36},
	language = {en},
	urldate = {2025-01-26},
	journal = {Advances in Neural Information Processing Systems},
	author = {Li, Puheng and Li, Zhong and Zhang, Huishuai and Bian, Jiang},
	month = dec,
	year = {2023},
	pages = {2097--2127},
}

@misc{achilli_memorization_2025,
	title = {Memorization and {Generalization} in {Generative} {Diffusion} under the {Manifold} {Hypothesis}},
	doi = {10.48550/arXiv.2502.09578},
	urldate = {2025-02-19},
	publisher = {arXiv},
	author = {Achilli, Beatrice and Ambrogioni, Luca and Lucibello, Carlo and Mézard, Marc and Ventura, Enrico},
	month = feb,
	year = {2025},
	note = {arXiv:2502.09578 [cond-mat]},
}

@inproceedings{ronneberger_u-net_2015,
	address = {Cham},
	title = {U-{Net}: {Convolutional} {Networks} for {Biomedical} {Image} {Segmentation}},
	isbn = {978-3-319-24574-4},
	shorttitle = {U-{Net}},
	doi = {10.1007/978-3-319-24574-4_28},
	language = {en},
	booktitle = {Medical {Image} {Computing} and {Computer}-{Assisted} {Intervention} – {MICCAI} 2015},
	publisher = {Springer International Publishing},
	author = {Ronneberger, Olaf and Fischer, Philipp and Brox, Thomas},
	editor = {Navab, Nassir and Hornegger, Joachim and Wells, William M. and Frangi, Alejandro F.},
	year = {2015},
	pages = {234--241},
}

@inproceedings{pidstrigach_score-based_2022,
	title = {Score-{Based} {Generative} {Models} {Detect} {Manifolds}},
	volume = {35},
	booktitle = {Advances in {Neural} {Information} {Processing} {Systems}},
	publisher = {Curran Associates, Inc.},
	author = {Pidstrigach, Jakiw},
	year = {2022},
	pages = {35852--35865},
}

@book{gardiner_stochastic_2009,
	address = {Berlin Heidelberg},
	edition = {4th ed},
	series = {Springer series in synergetics},
	title = {Stochastic methods: a handbook for the natural and social sciences},
	isbn = {978-3-642-08962-6},
	shorttitle = {Stochastic methods},
	language = {eng},
	number = {13},
	publisher = {Springer},
	author = {Gardiner, Crispin W.},
	year = {2009},
}

@misc{cui_precise_2025,
	title = {A precise asymptotic analysis of learning diffusion models: theory and insights},
	shorttitle = {A precise asymptotic analysis of learning diffusion models},
	doi = {10.48550/arXiv.2501.03937},
	urldate = {2025-01-13},
	publisher = {arXiv},
	author = {Cui, Hugo and Pehlevan, Cengiz and Lu, Yue M.},
	month = jan,
	year = {2025},
	note = {arXiv:2501.03937 [cs]},
}

@article{helton_applications_2018,
	title = {Applications of realizations (aka linearizations) to free probability},
	volume = {274},
	number = {1},
	urldate = {2025-01-28},
	journal = {Journal of Functional Analysis},
	author = {Helton, J. William and Mai, Tobias and Speicher, Roland},
	month = jan,
	year = {2018},
	pages = {1--79},
}

@inproceedings{adlam_neural_2020,
	title = {The {Neural} {Tangent} {Kernel} in {High} {Dimensions}: {Triple} {Descent} and a {Multi}-{Scale} {Theory} of {Generalization}},
	issn = {2640-3498},
	shorttitle = {The {Neural} {Tangent} {Kernel} in {High} {Dimensions}},
	language = {en},
	urldate = {2025-01-28},
	booktitle = {Proceedings of the 37th {International} {Conference} on {Machine} {Learning}},
	publisher = {PMLR},
	author = {Adlam, Ben and Pennington, Jeffrey},
	month = nov,
	year = {2020},
	pages = {74--84},
}

@inproceedings{yoon_diffusion_2023,
	title = {Diffusion {Probabilistic} {Models} {Generalize} when {They} {Fail} to {Memorize}},
	language = {en},
	urldate = {2025-01-13},
	booktitle = {{ICML} 2023 {Workshop} on {Structured} {Probabilistic} {Inference} and {Generative} {Modeling}},
	author = {Yoon, TaeHo and Choi, Joo Young and Kwon, Sehyun and Ryu, Ernest K.},
	month = jul,
	year = {2023},
}

@article{ambrogioni_search_2024,
	title = {In {Search} of {Dispersed} {Memories}: {Generative} {Diffusion} {Models} {Are} {Associative} {Memory} {Networks}},
	volume = {26},
	copyright = {https://creativecommons.org/licenses/by/4.0/},
	issn = {1099-4300},
	shorttitle = {In {Search} of {Dispersed} {Memories}},
	doi = {10.3390/e26050381},
	language = {en},
	number = {5},
	urldate = {2025-01-28},
	journal = {Entropy},
	author = {Ambrogioni, Luca},
	month = apr,
	year = {2024},
	pages = {381},
}

@article{vincent_connection_2011,
	title = {A {Connection} {Between} {Score} {Matching} and {Denoising} {Autoencoders}},
	volume = {23},
	issn = {0899-7667, 1530-888X},
	doi = {10.1162/NECO_a_00142},
	language = {en},
	number = {7},
	urldate = {2025-01-18},
	journal = {Neural Computation},
	author = {Vincent, Pascal},
	month = jul,
	year = {2011},
	pages = {1661--1674},
}

@inproceedings{song_score-based_2020,
	title = {Score-{Based} {Generative} {Modeling} through {Stochastic} {Differential} {Equations}},
	language = {en},
	urldate = {2025-01-28},
	booktitle = {International {Conference} on {Learning} {Representations}},
	author = {Song, Yang and Sohl-Dickstein, Jascha and Kingma, Diederik P. and Kumar, Abhishek and Ermon, Stefano and Poole, Ben},
	month = oct,
	year = {2020},
}

@inproceedings{raya_spontaneous_2023,
	title = {Spontaneous symmetry breaking in generative diffusion models},
	volume = {36},
	booktitle = {Advances in {Neural} {Information} {Processing} {Systems}},
	publisher = {Curran Associates, Inc.},
	author = {Raya, Gabriel and Ambrogioni, Luca},
	year = {2023},
	pages = {66377--66389},
}

@misc{kamb_analytic_2024,
	title = {An analytic theory of creativity in convolutional diffusion models},
	doi = {10.48550/arXiv.2412.20292},
	urldate = {2025-01-28},
	publisher = {arXiv},
	author = {Kamb, Mason and Ganguli, Surya},
	month = dec,
	year = {2024},
	note = {arXiv:2412.20292 [cs]},
}

@inproceedings{kadkhodaie_generalization_2023,
	title = {Generalization in diffusion models arises from geometry-adaptive harmonic representations},
	language = {en},
	urldate = {2025-01-28},
	booktitle = {The {Twelfth} {International} {Conference} on {Learning} {Representations}},
	author = {Kadkhodaie, Zahra and Guth, Florentin and Simoncelli, Eero P. and Mallat, Stéphane},
	month = oct,
	year = {2023},
}

@misc{gu_memorization_2023,
	title = {On {Memorization} in {Diffusion} {Models}},
	doi = {10.48550/arXiv.2310.02664},
	urldate = {2025-01-26},
	publisher = {arXiv},
	author = {Gu, Xiangming and Du, Chao and Pang, Tianyu and Li, Chongxuan and Lin, Min and Wang, Ye},
	month = oct,
	year = {2023},
	note = {arXiv:2310.02664 [cs]},
}

@inproceedings{goldt_gaussian_2022,
	title = {The {Gaussian} equivalence of generative models for learning with shallow neural networks},
	issn = {2640-3498},
	language = {en},
	urldate = {2025-01-22},
	booktitle = {Proceedings of the 2nd {Mathematical} and {Scientific} {Machine} {Learning} {Conference}},
	publisher = {PMLR},
	author = {Goldt, Sebastian and Loureiro, Bruno and Reeves, Galen and Krzakala, Florent and Mezard, Marc and Zdeborova, Lenka},
	month = apr,
	year = {2022},
	pages = {426--471},
}

@inproceedings{gerace_generalisation_2020,
	title = {Generalisation error in learning with random features and the hidden manifold model},
	issn = {2640-3498},
	language = {en},
	urldate = {2025-01-22},
	booktitle = {Proceedings of the 37th {International} {Conference} on {Machine} {Learning}},
	publisher = {PMLR},
	author = {Gerace, Federica and Loureiro, Bruno and Krzakala, Florent and Mezard, Marc and Zdeborova, Lenka},
	month = nov,
	year = {2020},
	pages = {3452--3462},
}

@misc{far_spectra_2006,
	title = {Spectra of large block matrices},
	urldate = {2024-10-23},
	publisher = {arXiv},
	author = {Far, Reza Rashidi and Oraby, Tamer and Bryc, Wlodzimierz and Speicher, Roland},
	month = oct,
	year = {2006},
	note = {arXiv:cs/0610045},
}

@inproceedings{cui_analysis_2023,
	title = {Analysis of {Learning} a {Flow}-based {Generative} {Model} from {Limited} {Sample} {Complexity}},
	language = {en},
	urldate = {2025-01-28},
	booktitle = {The {Twelfth} {International} {Conference} on {Learning} {Representations}},
	author = {Cui, Hugo and Krzakala, Florent and Vanden-Eijnden, Eric and Zdeborova, Lenka},
	month = oct,
	year = {2023},
}

@inproceedings{chen_sampling_2022,
	title = {Sampling is as easy as learning the score: theory for diffusion models with minimal data assumptions},
	shorttitle = {Sampling is as easy as learning the score},
	language = {en},
	urldate = {2025-01-28},
	booktitle = {The {Eleventh} {International} {Conference} on {Learning} {Representations}},
	author = {Chen, Sitan and Chewi, Sinho and Li, Jerry and Li, Yuanzhi and Salim, Adil and Zhang, Anru},
	month = sep,
	year = {2022},
}

@inproceedings{chen_score_2023,
	title = {Score {Approximation}, {Estimation} and {Distribution} {Recovery} of {Diffusion} {Models} on {Low}-{Dimensional} {Data}},
	issn = {2640-3498},
	language = {en},
	urldate = {2025-01-28},
	booktitle = {Proceedings of the 40th {International} {Conference} on {Machine} {Learning}},
	publisher = {PMLR},
	author = {Chen, Minshuo and Huang, Kaixuan and Zhao, Tuo and Wang, Mengdi},
	month = jul,
	year = {2023},
	pages = {4672--4712},
}

@phdthesis{bodin_random_2024,
	title = {Random matrix methods for high-dimensional machine learning models},
	doi = {10.5075/epfl-thesis-10524},
	language = {en},
	urldate = {2025-01-24},
	school = {EPFL},
	author = {Bodin, Antoine Philippe Michel},
	year = {2024},
}

@inproceedings{bodin_model_2021,
	title = {Model, sample, and epoch-wise descents: exact solution of gradient flow in the random feature model},
	volume = {34},
	shorttitle = {Model, sample, and epoch-wise descents},
	urldate = {2025-01-24},
	booktitle = {Advances in {Neural} {Information} {Processing} {Systems}},
	publisher = {Curran Associates, Inc.},
	author = {Bodin, Antoine and Macris, Nicolas},
	year = {2021},
	pages = {21605--21617},
}

@inproceedings{benton_nearly_2023,
	title = {Nearly d-{Linear} {Convergence} {Bounds} for {Diffusion} {Models} via {Stochastic} {Localization}},
	language = {en},
	urldate = {2025-01-28},
	booktitle = {The {Twelfth} {International} {Conference} on {Learning} {Representations}},
	author = {Benton, Joe and Bortoli, Valentin De and Doucet, Arnaud and Deligiannidis, George},
	month = oct,
	year = {2023},
}

@misc{achilli_losing_2024,
	title = {Losing dimensions: {Geometric} memorization in generative diffusion},
	shorttitle = {Losing dimensions},
	urldate = {2024-11-17},
	publisher = {arXiv},
	author = {Achilli, Beatrice and Ventura, Enrico and Silvestri, Gianluigi and Pham, Bao and Raya, Gabriel and Krotov, Dmitry and Lucibello, Carlo and Ambrogioni, Luca},
	month = oct,
	year = {2024},
	note = {arXiv:2410.08727},
}

@inproceedings{han_neural_2023,
	title = {Neural {Network}-{Based} {Score} {Estimation} in {Diffusion} {Models}: {Optimization} and {Generalization}},
	shorttitle = {Neural {Network}-{Based} {Score} {Estimation} in {Diffusion} {Models}},
	language = {en},
	urldate = {2025-01-17},
	booktitle = {The {Twelfth} {International} {Conference} on {Learning} {Representations}},
	author = {Han, Yinbin and Razaviyayn, Meisam and Xu, Renyuan},
	month = oct,
	year = {2023},
}

@inproceedings{song_generative_2019,
	title = {Generative {Modeling} by {Estimating} {Gradients} of the {Data} {Distribution}},
	volume = {32},
	urldate = {2024-05-22},
	booktitle = {Advances in {Neural} {Information} {Processing} {Systems}},
	publisher = {Curran Associates, Inc.},
	author = {Song, Yang and Ermon, Stefano},
	year = {2019},
}

@inproceedings{somepalli_diffusion_2023,
	title = {Diffusion {Art} or {Digital} {Forgery}? {Investigating} {Data} {Replication} in {Diffusion} {Models}},
	issn = {2575-7075},
	shorttitle = {Diffusion {Art} or {Digital} {Forgery}?},
	urldate = {2025-01-13},
	booktitle = {2023 {IEEE}/{CVF} {Conference} on {Computer} {Vision} and {Pattern} {Recognition} ({CVPR})},
	author = {Somepalli, Gowthami and Singla, Vasu and Goldblum, Micah and Geiping, Jonas and Goldstein, Tom},
	month = jun,
	year = {2023},
	pages = {6048--6058},
}

@inproceedings{sohl-dickstein_deep_2015,
	title = {Deep {Unsupervised} {Learning} using {Nonequilibrium} {Thermodynamics}},
	issn = {1938-7228},
	language = {en},
	urldate = {2025-01-10},
	booktitle = {Proceedings of the 32nd {International} {Conference} on {Machine} {Learning}},
	publisher = {PMLR},
	author = {Sohl-Dickstein, Jascha and Weiss, Eric and Maheswaranathan, Niru and Ganguli, Surya},
	month = jun,
	year = {2015},
	pages = {2256--2265},
}

@article{shah_learning_2023,
	title = {Learning {Mixtures} of {Gaussians} {Using} the {DDPM} {Objective}},
	volume = {36},
	language = {en},
	urldate = {2025-01-14},
	journal = {Advances in Neural Information Processing Systems},
	author = {Shah, Kulin and Chen, Sitan and Klivans, Adam},
	month = dec,
	year = {2023},
	pages = {19636--19649},
}

@inproceedings{rahimi_random_2007,
	title = {Random {Features} for {Large}-{Scale} {Kernel} {Machines}},
	volume = {20},
	urldate = {2025-01-13},
	booktitle = {Advances in {Neural} {Information} {Processing} {Systems}},
	publisher = {Curran Associates, Inc.},
	author = {Rahimi, Ali and Recht, Benjamin},
	year = {2007},
}

@article{hu_universality_2023,
	title = {Universality {Laws} for {High}-{Dimensional} {Learning} {With} {Random} {Features}},
	volume = {69},
	number = {3},
	urldate = {2025-01-13},
	journal = {IEEE Transactions on Information Theory},
	author = {Hu, Hong and Lu, Yue M.},
	month = mar,
	year = {2023},
	note = {Conference Name: IEEE Transactions on Information Theory},
	pages = {1932--1964},
}

@inproceedings{ho_denoising_2020,
	title = {Denoising {Diffusion} {Probabilistic} {Models}},
	volume = {33},
	urldate = {2024-05-22},
	booktitle = {Advances in {Neural} {Information} {Processing} {Systems}},
	publisher = {Curran Associates, Inc.},
	author = {Ho, Jonathan and Jain, Ajay and Abbeel, Pieter},
	year = {2020},
	pages = {6840--6851},
}

@article{haussmann_time_1986,
	title = {Time {Reversal} of {Diffusions}},
	volume = {14},
	number = {4},
	urldate = {2025-01-10},
	journal = {The Annals of Probability},
	publisher = {Institute of Mathematical Statistics},
	author = {Haussmann, U. G. and Pardoux, E.},
	month = oct,
	year = {1986},
	pages = {1188--1205},
}

@article{bortoli_convergence_2022,
	title = {Convergence of denoising diffusion models under the manifold hypothesis},
	issn = {2835-8856},
	language = {en},
	urldate = {2025-01-14},
	journal = {Transactions on Machine Learning Research},
	author = {Bortoli, Valentin De},
	month = aug,
	year = {2022},
}

@article{biroli_dynamical_2024,
	title = {Dynamical regimes of diffusion models},
	volume = {15},
	copyright = {2024 The Author(s)},
	issn = {2041-1723},
	doi = {10.1038/s41467-024-54281-3},
	language = {en},
	number = {1},
	urldate = {2025-01-08},
	journal = {Nature Communications},
	publisher = {Nature Publishing Group},
	author = {Biroli, Giulio and Bonnaire, Tony and de Bortoli, Valentin and Mézard, Marc},
	month = nov,
	year = {2024},
	pages = {9957},
}

@article{anderson_reverse-time_1982,
	title = {Reverse-time diffusion equation models},
	volume = {12},
	number = {3},
	urldate = {2025-01-10},
	journal = {Stochastic Processes and their Applications},
	author = {Anderson, Brian D. O.},
	month = may,
	year = {1982},
	pages = {313--326},
}

@inproceedings{chen_improved_2023,
	address = {Honolulu, Hawaii, USA},
	series = {{ICML}'23},
	title = {Improved analysis of score-based generative modeling: user-friendly bounds under minimal smoothness assumptions},
	volume = {202},
	shorttitle = {Improved analysis of score-based generative modeling},
	urldate = {2025-01-13},
	booktitle = {Proceedings of the 40th {International} {Conference} on {Machine} {Learning}},
	publisher = {JMLR.org},
	author = {Chen, Hongrui and Lee, Holden and Lu, Jianfeng},
	month = jul,
	year = {2023},
	pages = {4735--4763},
}

@inproceedings{carlini_extracting_2023,
	address = {USA},
	series = {{SEC} '23},
	title = {Extracting training data from diffusion models},
	isbn = {978-1-939133-37-3},
	urldate = {2025-01-13},
	booktitle = {Proceedings of the 32nd {USENIX} {Conference} on {Security} {Symposium}},
	publisher = {USENIX Association},
	author = {Carlini, Nicholas and Hayes, Jamie and Nasr, Milad and Jagielski, Matthew and Sehwag, Vikash and Tramèr, Florian and Balle, Borja and Ippolito, Daphne and Wallace, Eric},
	month = aug,
	year = {2023},
	pages = {5253--5270},
}

\clearpage
\appendix
\thispagestyle{empty}

\onecolumn
\aistatstitle{Supplementary Materials}

\section{Proofs}
We first derive the finite-dimensional rank-constrained optimizer. We then express its training and test errors as spectral integrals and compute the required limiting transforms using a linear pencil. The moment calculations underlying this reduction are collected in Section~\ref{appndx:matrices}.
\subsection{Proof of Theorem~\ref{thm:rank_constrained_dsm}}
Let $\phi(x) = \act{\rat{W}{d}x}$. Consider $s(x) = -x+A\phi(x)$ and $x\sim\cN{0,\cC}$. Let $y_i = ax_i+\sqrt{h}z$. We have the following minimization problem
\begin{equation*}
   \min_{\text{rank}(A)\le r} \cL(A) = \sum_{i=1}^n\shortexpect_z{\norm{\sqrt{h}(A\phi(y_i)-y_i)+z}^2}.
\end{equation*}
Condition on the training samples and $W$, and assume $U\succ0$ as in Theorem~\ref{thm:rank_constrained_dsm}. Differentiating the quadratic objective gives the normal equation $A_{\text{LS}}U=\left(V_{\phi y}-\frac{1}{\sqrt{h}}V_{\phi z}\right)^T$, so the unconstrained optimizer is
\begin{equation}\label{eq:app_ls_optimizer}
    A_{\text{LS}} = \underbrace{\left(V_{\phi y}-\frac{1}{\sqrt{h}}V_{\phi z}\right)^T}_{:=V^T}U^{-1}, 
\end{equation}
where
\begin{equation}\label{eq:app_empirical_moments}
\begin{aligned}
    U &= \frac{1}{n}\sum_{i=1}^n\bE{z}{\phi(y_i)\phi(y_i)^T},\\
    V_{\phi z} &= \frac{1}{n}\sum_{i=1}^n\bE{z}{\phi(y_i)z^T},\\
    V_{\phi y} &= \frac{1}{n}\sum_{i=1}^n\bE{z}{\phi(y_i)y_i^T}.
\end{aligned}
\end{equation}
To impose the rank constraint, expand the objective around $A_{\text{LS}}$. The mixed term vanishes by the normal equation associated with~\eqref{eq:app_ls_optimizer}:
\begin{align*}
   \cL(A) &= \sum_{i=1}^n\shortexpect_z{\norm{\sqrt{h}(A-A_{\text{LS}})\phi(ax_i+\sqrt{h}z)+\sqrt{h}(A_{\text{LS}}\phi(ax_i+\sqrt{h}z)-(ax_i+\sqrt{h}z))+z}^2},\\
   &= \sum_{i=1}^n\shortexpect_z\norm{\sqrt{h}(A-A_{\text{LS}})\phi(ax_i+\sqrt{h}z)}^2+\sum_{i=1}^n\shortexpect_z\norm{\sqrt{h}(A_{\text{LS}}\phi(ax_i+\sqrt{h}z)-(ax_i+\sqrt{h}z))+z}^2,\\
\end{align*}
The residual term is independent of $A$. Dropping this term and positive multiplicative constant $h$, and using~\eqref{eq:app_empirical_moments}, reduces the optimization to a matrix approximation problem:
\begin{align*}
    \arg \min_{\text{rank}(A)\le r} \cL(A) &= \arg \min_{\text{rank}(A)\le r} \sum_{i=1}^n\shortexpect_z\norm{(A-A_{\text{LS}})\phi(ax_i+\sqrt{h}z)}^2,\\
    &= \arg \min_{\text{rank}(A)\le r} \norm{(A-A_{\text{LS}})U^{1/2}}.
\end{align*}
Let $A_{\text{LS}}U^{1/2} = S\Sigma R^T$ be the singular value decomposition. The Eckart--Young theorem in the Frobenius norm therefore gives 
\begin{equation*}
    \hat{A}_r = S\Sigma_r R^T U^{-1/2},
\end{equation*}
where $\Sigma_r$ retains the $r$ largest singular values and sets the others to zero.

\subsection{Proof of Theorem~\ref{thm:limiting_test_error}}\label{app:rank_constrained_learning}

The proof proceeds by reducing the losses to weighted resolvent traces, obtaining their limiting values from~\eqref{eqn:fixed_point_system}, and recovering the rank-truncated errors by Stieltjes inversion. Barred moments ($\bar U, \bar V_{\phi y}, \bar V_{\phi z}$) average over an independent test sample and its noise; empirical moments condition on the clean training samples.

\subsubsection{Test Error}
We assume that $\bE{}{\act{g}}=0$. The data covariance is diagonal with first $D$ diagonal entries equal to 1 and the remaining $d-D$ entries equal to 0. Let $\sigma_0^2 = \psi_D$, the normalized trace of the covariance matrix.
Accordingly, write $X = [X_1^T, 0]^T$, where $X_1$ is a $D\times n$ Gaussian random matrix, and partition $W=[W_1, W_2]$ into its data-subspace and normal-space blocks. Expanding the squared residual separates the test error into a quadratic term, a cross term, and an estimator-independent term. With $y = ax+\sqrt{h}z$, we have
\begin{equation}\label{eq:app_test_expansion}
\begin{aligned}
    \cE_{\text{test}}(\hat{A}_r) &= \frac{1}{d}\shortexpect_{x\sim P_0,z\sim\cN{0,1}}{\norm{\sqrt{h}(\hat{A}_r\phi(y)-y)+z}^2},\\
    &= \frac{h}{d}\tr{\hat{A}_r^T\hat{A}_r\bar{U}}+\frac{2h}{d}\tr{\hat{A}_r\left(\frac{\bar{V}_{\phi z}}{\sqrt{h}}-\bar{V}_{\phi y}\right)} +\frac{h}{d}\tr{a^2\cC+hI}+\left(1-2h\right),
\end{aligned}
\end{equation}
where $\bar{U} = \bE{x,z}{\phi(ax+\sqrt{h}z)\phi(ax+\sqrt{h}z)^\top}$, $\bar{V}_{\phi y} = \bE{x,z}{\phi(ax+\sqrt{h}z)(ax+\sqrt{h}z)^\top}$, and $\bar{V}_{\phi z} = \bE{x,z}{\phi(ax+\sqrt{h}z)z^\top}$.

We need to compute $V_{\phi y},V_{\phi z},U,\bar{V}_{\phi y},\bar{V}_{\phi z},\bar{U}$ in order to get an expression for $\cE_{\text{test}}$. Note that, in order to derive an asymptotically precise expression for $\cE_{\text{test}}$, it suffices to obtain these moment matrices up to leading order. In particular, we can neglect $O(1/d^{1-\epsilon})$ terms entry-wise in these matrices, for $0\le\epsilon<1/2$ and for the diagonal terms $\epsilon$ can be in $[0,1)$. This can be justifies as follows. By union bound, such an error matrix $\Delta$ will have Frobenius $\norm{\Delta}_F = O(d^\epsilon)$ with probability tending to $1$ faster than any power. In the expressions for test and train errors, these error matrices typically appears as $\frac{1}{d}\tr{\Delta A}$, where $A$ is a matrix with operator norm $\bigO{1}$ w.r.t. $d$. Using standard inequalities, we can upper bound $\frac{1}{d}\tr{\Delta A}$ as
\begin{align*}
    \frac{1}{d}\tr{\Delta A} &\le\frac{1}{d}\norm{\Delta}_\ast\norm{A}_{\text{op}}\\
    &\le \frac{1}{d}\sqrt{\text{rank}(\Delta)}\norm{\Delta}_F\norm{A}_{\text{op}}\\
    &= \bigO{\frac{1}{\sqrt{d}}}\norm{\Delta}_F = O(d^{\epsilon-1/2}). 
\end{align*}
where the last equality holds with probability tending to $1$ faster than any inverse power. Thus, asymptotically these error terms vanish with high probability. 
Section~\ref{appndx:matrices} develops the moment calculations underlying the following asymptotic representations, which we use in the spectral reduction below:

\begin{equation}\label{eq:app_moment_representations}
\begin{aligned}
    \bar{V}_{\phi y} &= \frac{\mu_1(\sigma)}{\sigma}\rat{W\tilde{\cC}}{d},\\
    \bar{V}_{\phi z} &= \sqrt{h}\frac{\mu_1(\sigma)}{\sigma}\rat{W}{d},\\
    V_{\phi y} &= a\rat{H}{n}\rat{X^T}{n}+\sqrt{h}\bar{V}_{\phi z},\\
    V_{\phi z} &= \sqrt{h}\frac{\mu_1(\sigma)}{\sigma}\rat{W}{d} =  \bar{V}_{\phi z},\\
    \bar{U} &=  \frac{\mu_1^2(\sigma)}{\sigma^2}\rat{W}{d}\tilde{\cC}\rat{W^T}{d} + v^2(\sigma) I_p \;,\\
    U &= \rat{H}{n}\rat{H^T}{n}+h\frac{\mu_1^2(\sigma)}{\sigma^2}\rat{W}{d}\rat{W^T}{d}+s^2I_p \;,
\end{aligned}
\end{equation}
where $\tilde{\cC} = a^2\cC+hI$, $H= \varrho_0\left(\rat{W}{d}X\right)$, and 
\begin{align}
s^2 &= \norm{\varrho(\sigma\;\cdot)}^2-c(a^2\sigma_0^2/\sigma^2,\sigma)-h\frac{\mu_1^2(\sigma)}{\sigma^2} \;,\label{eqn:s2}\\
v^2(\sigma) &= \norm{\varrho(\sigma\;\cdot)}^2-\mu_1^2(\sigma).\label{eqn:v2}
\end{align}
Here, for standard Gaussian random variables $(u,v)$ with correlation $\gamma$, we define $c(\gamma,\kappa) = \bE{u,v}{\varrho(\kappa u)\varrho(\kappa v)}$.

\begin{assumption}[Gaussian equivalence]\label{asmptn:GEP}
Set $\mu=\mu_1(\sigma)/\sigma$ and define the Gaussian surrogate
\begin{equation*}
    H^{\mathrm G} = a\mu\frac{WX}{\sqrt d}+v_0\Omega\;,
\end{equation*}
where $\Omega\in\R^{p\times n}$ has i.i.d. $\cN{0,1}$ entries and is independent of $W$ and $X$, and
\begin{align}
    v_0^2 &= c(a^2\sigma_0^2/\sigma^2,\sigma)-\left(a\sigma_0\frac{\mu_1(\sigma)}{\sigma}\right)^2\;.\label{eqn:v02}
\end{align}
For $F\in\{H,H^{\mathrm G}\}$, define
\begin{align*}
    U_F &= \frac{FF^T}{n}+h\mu^2\frac{WW^T}{d}+s^2I_p,\\
    V_F &= \frac{aFX^T}{n}-a^2\mu\frac{W}{\sqrt d},\\
    Q_F(z) &= (V_FV_F^T-zU_F)^{-1},
\end{align*}
and the normalized traces
\begin{align*}
    E_{0,F}(z) &= \frac{1}{d}\tr{V_F^TQ_F(z)V_F},\\
    E_{1,F}(z) &= \frac{1}{d}\tr{Q_F(z)},\\
    E_{2,F}(z) &= \frac{1}{d}\tr{Q_F(z)\frac{W_1W_1^T}{d}},\\
    E_{3,F}(z) &= \frac{1}{d}\tr{Q_F(z)\frac{W_2W_2^T}{d}}.
\end{align*}
Assume $s^2>0$ and that, in the proportional limit of Theorem~\ref{thm:limiting_test_error}, for every fixed $z\in\mathbb C$ with $\Im z>0$,
\[
    E_{j,H}(z)-E_{j,H^{\mathrm G}}(z)
    \xrightarrow{\mathbb P}0,
    \qquad j\in\{0,1,2,3\}.
\]
\end{assumption}

Conditioning on this assumption, permits replacing $H$ by $H^{\mathrm G}$ in these four limiting traces, with the same $W$ and $X$ in both models.

Let $\tilde{\cC} = a^2\cC+hI$. Note that $\tilde{\cC}$ is diagonal with first $D$ diagonal entries equal to 1 and the remaining $d-D$ entries equal to $h$.
To incorporate rank truncation into~\eqref{eq:app_test_expansion}, introduce the two Gram matrices associated with the singular value decomposition of $V^TU^{-1/2}$:
\begin{equation}\label{eq:app_gram_matrices}
\begin{aligned}
    M &= S\Sigma\Sigma^TS^T = A_{\text{LS}}UA_{\text{LS}}^T = V^TU^{-1}V,\\
    \tilde{M} &= R\Sigma^T\Sigma R^T = U^{1/2}A_{\text{LS}}^TA_{\text{LS}}U^{1/2} = U^{-1/2}VV^TU^{-1/2}.
\end{aligned}
\end{equation}
Recall that $\psi_p=p/d$, $\psi_n=n/d$, $\psi_D=D/d$, and $\psi_r=r/d$. The matrices in~\eqref{eq:app_gram_matrices} have dimensions $d\times d$ and $p\times p$, respectively, and share the same nonzero eigenvalues. Let $\Gamma_{\psi_r}$ be the threshold such that $\psi_r$ fraction of eigenvalues of $M$ are above $\Gamma_{\psi_r}$. We can use Stone's integration representation to get the rank-truncated versions of $M$ and $\tilde{M}$. Let $S_i$ be the $i$th column of $S$, and $s_i$ be the corresponding singular value:
\begin{equation}\label{eq:app_feature_truncation}
\begin{aligned}
    S\Sigma_r\Sigma_r^T S^T &= \sum_{i\le r} s_i^2S_iS_i^T, \\
    &= \frac{1}{\pi}\lim_{K\to\infty}\lim_{\eta\downarrow0}\int_{\Gamma_{\psi_r}}^K dx\; x\Im(M-(x+i\eta)I)^{-1}.
\end{aligned}
\end{equation}
Similarly,
\begin{equation}\label{eq:app_output_truncation}
\begin{aligned}
    R\Sigma_r^T\Sigma_r R^T = \frac{1}{\pi}\lim_{K\to\infty}\lim_{\eta\downarrow0}\int_{\Gamma_{\psi_r}}^K dx\; x\Im(\tilde M-(x+i\eta)I)^{-1}.
\end{aligned}
\end{equation}
Using~\eqref{eq:app_gram_matrices}, the output-space resolvent can be expressed through the feature-space matrix $VV^T-zU$:
\begin{equation}\label{eq:app_resolvent_identity}
\begin{aligned}
    (M-zI)^{-1} &= (V^TU^{-1}V-zI)^{-1},\\
    &= \frac{1}{z}\left((V^TU^{-1}V-zI)^{-1}V^TU^{-1}V-I\right),\\
    &= \frac{1}{z}\left(V^T(VV^T-zU)^{-1}V-I\right),\\
\end{aligned}
\end{equation}
Let $Q(z):= (VV^T-zU)^{-1}$. Substituting~\eqref{eq:app_moment_representations} into the definition of $V$ and using $h-1=-a^2$ gives 
\begin{align*}
    V &= V_{\phi y}-\frac{1}{\sqrt{h}}V_{\phi z}\\
    &= a\rat{H}{n}\rat{X^T}{n}+(h-1)\frac{\mu_1(\sigma)}{\sigma}\rat{W}{d},\\
    &= a\rat{H}{n}\rat{X^T}{n}-a^2\frac{\mu_1(\sigma)}{\sigma}\rat{W}{d}.
\end{align*}
Because $X^T(\cC-I)=0$, multiplication by $\cC-I$ removes the data-dependent term:
\begin{align*}
    V(\cC-I) &=  a^2\frac{\mu_1(\sigma)}{\sigma}\rat{[0, W_2]}{d}.
\end{align*}
and
\begin{equation}\label{eq:app_cross_identity}
\begin{aligned}
    \frac{\bar{V}_{\phi z}}{\sqrt{h}}-\bar{V}_{\phi y} &=  \frac{\mu_1(\sigma)}{\sigma}\rat{W(I-\tilde{\cC})}{d},\\
    &=  (1-h)\frac{\mu_1(\sigma)}{\sigma}\rat{[0, W_2]}{d},\\
    &=  a^2\frac{\mu_1(\sigma)}{\sigma}\rat{[0, W_2]}{d},\\
    &= V(\cC-I).
\end{aligned}
\end{equation}
Next, for the quadratic term in~\eqref{eq:app_test_expansion}, we substitute the optimizer and apply~\eqref{eq:app_feature_truncation}. Moving $U^{-1/2}$ through the resolvent produces $Q$; inserting~\eqref{eq:app_moment_representations} then separates the contributions of the two blocks of $W$:
\begin{align*}
    \frac{1}{d}\tr{\hat{A}_r^T\hat{A}_r\bar{U}} &= \frac{1}{d}\tr{U^{-1/2}R\Sigma_r^T\Sigma_r R^TU^{-1/2}\bar{U}},\\
    &= \frac{1}{\pi}\lim_{K\to\infty}\lim_{\eta\downarrow0}\int_{\Gamma_{\psi_r}}^K dx\; x\Im\frac{1}{d}\tr{U^{-1/2}(\tilde{M}-(x+i\eta)I)^{-1}U^{-1/2}\bar{U}},\\
    &= \frac{1}{\pi}\lim_{K\to\infty}\lim_{\eta\downarrow0}\int_{\Gamma_{\psi_r}}^K dx\; x\Im\frac{1}{d}\tr{U^{-1/2}(U^{-1/2}VV^TU^{-1/2}-(x+i\eta)I)^{-1}U^{-1/2}\bar{U}},\\
    &= \frac{1}{\pi}\lim_{K\to\infty}\lim_{\eta\downarrow0}\int_{\Gamma_{\psi_r}}^K dx\; x\Im\frac{1}{d}\tr{(VV^T-(x+i\eta)U)^{-1}\bar{U}},\\
    &= \frac{1}{\pi}\lim_{K\to\infty}\lim_{\eta\downarrow0}\int_{\Gamma_{\psi_r}}^K dx\; x\Im\frac{1}{d}\tr{Q(x+i\eta)\left(\frac{\mu_1^2(\sigma)}{\sigma^2}\rat{W}{d}\tilde{\cC}\rat{W^T}{d} + v^2(\sigma) I_p\right)},\\
    &= \frac{1}{\pi}\lim_{K\to\infty}\lim_{\eta\downarrow0}\int_{\Gamma_{\psi_r}}^K dx\; x\Im\frac{1}{d}\tr{Q(x+i\eta)\left(\frac{\mu_1^2(\sigma)}{\sigma^2}\left(\rat{W_1}{d}\rat{W_1^T}{d}+h\rat{W_2}{d}\rat{W_2^T}{d}\right) + v^2(\sigma) I_p\right)},\\
    &= \frac{1}{\pi}\lim_{K\to\infty}\lim_{\eta\downarrow0}\int_{\Gamma_{\psi_r}}^K dx\; x\Im\left(\frac{\mu_1^2(\sigma)}{\sigma^2}E_{2}(x+i\eta)+\frac{h\mu_1^2(\sigma)}{\sigma^2}E_{3}(x+i\eta)+{v^2(\sigma)}E_{1}(x+i\eta)\right),
\end{align*}
where
\begin{equation}\label{eq:app_weighted_traces}
\begin{aligned}
    E_{1}(z) &= \frac{1}{d}\tr{Q(z)},\\
    E_{2}(z) &= \frac{1}{d}\tr{\rat{W_1^T}{d}Q(z)\rat{W_1}{d}},\\
    E_{3}(z) &= \frac{1}{d}\tr{\rat{W_2^T}{d}Q(z)\rat{W_2}{d}}.
\end{aligned}
\end{equation}
For the cross term in~\eqref{eq:app_test_expansion}, use~\eqref{eq:app_cross_identity}, followed by~\eqref{eq:app_output_truncation} and~\eqref{eq:app_resolvent_identity}.
\begin{align*}
    \frac{1}{d}\tr{\hat{A}_r\left(\frac{\bar{V}_{\phi z}}{\sqrt{h}}-\bar{V}_{\phi y}\right)} &= \frac{1}{d}\tr{S\Sigma_r R^TU^{-1/2}V(\cC-I)},\\
    &= \frac{1}{d}\tr{S\Sigma_r R^TU^{1/2}A_{\text{LS}}^T(\cC-I)},\\
    &= \frac{1}{d}\tr{S\Sigma_r \Sigma^TS^T(\cC-I)} = \frac{1}{d}\tr{S\Sigma_r \Sigma_r^TS^T(\cC-I)},\\
    &= \frac{1}{\pi}\lim_{K\to\infty}\lim_{\eta\downarrow0}\int_{\Gamma_{\psi_r}}^K dx\; x\Im\frac{1}{d}\tr{(M-(x+i\eta)I)^{-1}(\cC-I)},\\
    &= \frac{1}{\pi}\lim_{K\to\infty}\lim_{\eta\downarrow0}\int_{\Gamma_{\psi_r}}^K dx\; x\Im\frac{1}{d}\tr{\frac{1}{(x+i\eta)}V^T(VV^T-(x+i\eta)U)^{-1}V(\cC-I)},\\
    &= -a^4\frac{\mu_1^2(\sigma)}{\sigma^2}\frac{1}{\pi}\lim_{K\to\infty}\lim_{\eta\downarrow0}\int_{\Gamma_{\psi_r}}^K dx\; x\Im\frac{1}{d}\tr{\frac{1}{(x+i\eta)}\rat{W_2^T}{d}Q(x+i\eta)\rat{W_2}{d}},\\
    &= -a^4\frac{\mu_1^2(\sigma)}{\sigma^2}\frac{1}{\pi}\lim_{K\to\infty}\lim_{\eta\downarrow0}\int_{\Gamma_{\psi_r}}^K dx\; x\Im\frac{E_{3}(x+i\eta)}{(x+i\eta)}.\\
\end{align*}
The test error therefore requires the limiting weighted traces in~\eqref{eq:app_weighted_traces}. Determining $\Gamma_{\psi_r}$ also requires the spectrum of $M$. Taking the normalized trace of~\eqref{eq:app_resolvent_identity} gives its Stieltjes transform:
\begin{equation}\label{eq:app_stieltjes_trace}
    \textbf{s}_{M}(z) = \frac{1}{d}\tr{(M-zI)^{-1}} = \frac{1}{zd}\tr{V^TQ(z)V}-\frac{1}{z} = \frac{1}{z}(E_0(z)-1),
\end{equation}
where 
\[E_0(z) = \frac{1}{d}\tr{V^TQ(z)V}.\]
To summarize, we need the four traces $E_0, E_1, E_2, E_3$ to compute the test error.
After expressing training error also in terms of these traces, we later construct a linear pencil that encodes the four traces.
\subsubsection{Train Error}
The training error has the same quadratic structure as~\eqref{eq:app_test_expansion}, with empirical moments replacing population moments. These terms simplify because the same empirical moments define $\hat A_r$.
\begin{equation}\label{eq:app_train_expansion}
\begin{aligned}
    \cE_{\text{train}}(\hat{A}_r) &= \frac{1}{dn}\sum_{i=1}^n\shortexpect_{z}{\norm{\sqrt{h}(\hat{A}_r\phi(y_i)-y_i)+z}^2},\\
    &= \frac{h}{d}\tr{\hat{A}_r^T\hat{A}_r U}+\frac{2h}{d}\tr{\hat{A}_r\left(\frac{V_{\phi z}}{\sqrt{h}}-V_{\phi y}\right)} +h\sigma^2+\left(1-2h\right),
\end{aligned}
\end{equation}

\begin{align*}
    \frac{1}{d}\tr{\hat{A}_r^T\hat{A}_rU} &= \frac{1}{d}\tr{U^{-1/2}R\Sigma_r^T\Sigma_r R^TU^{-1/2}U},\\
    &= \frac{1}{\pi}\lim_{K\to\infty}\lim_{\eta\downarrow0}\int_{\Gamma_{\psi_r}}^K dx\; x\Im\frac{1}{d}\tr{(\tilde{M}-(x+i\eta)I)^{-1}},\\
    &= \frac{1}{\pi}\lim_{K\to\infty}\lim_{\eta\downarrow0}\int_{\Gamma_{\psi_r}}^K dx\; x\Im\frac{1}{d}\tr{(U^{-1/2}VV^TU^{-1/2}-(x+i\eta)I)^{-1}},\\
    &= \frac{1}{\pi}\lim_{K\to\infty}\lim_{\eta\downarrow0}\int_{\Gamma_{\psi_r}}^K dx\; x\Im\frac{1}{d}\left(\tr{(M-(x+i\eta)I)^{-1}}-(p-d)\frac{1}{x+i\eta}\right),\\
    &= \frac{1}{\pi}\lim_{K\to\infty}\lim_{\eta\downarrow0}\int_{\Gamma_{\psi_r}}^K dx\; x\Im\frac{1}{d}\tr{(M-(x+i\eta)I)^{-1}}.
\end{align*}
The extra term involving $p-d$ accounts for the different numbers of zero eigenvalues of $M$ and $\tilde M$. It contributes no mass above the positive cutoff. For the cross term, use $V^T=A_{\mathrm{LS}}U$ from~\eqref{eq:app_ls_optimizer}:
\begin{align*}
    \frac{1}{d}\tr{\hat{A}_r\left(\frac{V_{\phi z}}{\sqrt{h}}-V_{\phi y}\right)} &= -\frac{1}{d}\tr{S\Sigma_r R^TU^{-1/2}V},\\
    &= -\frac{1}{d}\tr{S\Sigma_r R^TU^{1/2}A_{\text{LS}}^T},\\
    &= -\frac{1}{d}\tr{S\Sigma_r \Sigma^TS^T} = -\frac{1}{d}\tr{S\Sigma_r \Sigma_r^TS^T},\\
    &= -\frac{1}{\pi}\lim_{K\to\infty}\lim_{\eta\downarrow0}\int_{\Gamma_{\psi_r}}^K dx\; x\Im\frac{1}{d}\tr{(M-(x+i\eta)I)^{-1}}.\\
\end{align*}

Combining the two contributions with their coefficients in~\eqref{eq:app_train_expansion} leaves a negative upper-tail spectral integral. Unlike the test error, the training error requires only the eigenvalue distribution of $M$.

\subsubsection{Linear Pencil}
The pencil replaces products of random matrices in $Q(z)$ by a larger matrix with blocks linear in $W_1,W_2,X_1$, and $\Omega$. Selected inverse blocks recover the traces in~\eqref{eq:app_weighted_traces} and~\eqref{eq:app_stieltjes_trace}.
Set $\mu:=\mu_1(\sigma)/\sigma$. With block sizes $(p,D,d-D,n,D,n,D,D,d-D)$, the pencil is
\begin{equation}\label{eqn:linear_pencil}
\resizebox{\textwidth}{!}{$\displaystyle
L(z)=
\begin{pmatrix}
I_p & -\frac{h\mu^2 W_1}{\sqrt{d}} & \frac{a^4\mu^2 W_2}{z\sqrt{d}} & -\frac{v_0\Omega}{\sqrt{n}} & -\frac{a^2\mu W_1}{z\sqrt{d}} & \frac{av_0\Omega}{z\sqrt{n}} & -\frac{a\mu W_1}{\sqrt{d}} & \frac{a^2\mu W_1}{z\sqrt{d}} & -\frac{h\mu^2 W_2}{\sqrt{d}} \\
\frac{W_1^T}{s^2\sqrt{d}} & I_D & 0 & 0 & 0 & 0 & 0 & 0 & 0 \\
\frac{W_2^T}{s^2\sqrt{d}} & 0 & I_{d-D} & 0 & 0 & 0 & 0 & 0 & 0 \\
\frac{v_0\Omega^T}{s^2\sqrt{n}} & -\frac{a\mu X_1^T}{\sqrt{n}} & 0 & I_n & 0 & 0 & 0 & 0 & 0 \\
-\frac{a^2\mu W_1^T}{s^2\sqrt{d}} & 0 & 0 & -\frac{aX_1}{\sqrt{n}} & I_D & 0 & 0 & 0 & 0 \\
0 & 0 & 0 & 0 & -\frac{X_1^T}{\sqrt{n}} & I_n & 0 & 0 & 0 \\
0 & 0 & 0 & -\frac{X_1}{\sqrt{n}} & 0 & 0 & I_D & 0 & 0 \\
0 & 0 & 0 & 0 & 0 & -\frac{X_1}{\sqrt{n}} & 0 & I_D & 0 \\
\frac{W_2^T}{s^2\sqrt{d}} & 0 & 0 & 0 & 0 & 0 & 0 & 0 & I_{d-D}
\end{pmatrix}.
$}
\end{equation}
Here, each zero has the corresponding block dimensions. Eliminating the auxiliary blocks of~\eqref{eqn:linear_pencil} by block inversion recovers $Q(z)$ and gives the following identities. Block indices in this display start at one: 
\begin{align*}
    E_0(z) &= \left(1-\frac{1}{D}\tr{L^{-1}_{55}}\right)\psi_D+\left(1-\frac{1}{d-D}\tr{L^{-1}_{33}}\right)(1-\psi_D),\\
    E_1(z) &= -\frac{1}{ds^2z}\tr{L^{-1}_{11}},\\
    E_2(z) &= -\frac{1}{da^2\mu}\tr{L^{-1}_{28}},\\
    E_3(z) &= \frac{1}{a^4\mu^2}\left(1-\frac{1}{d-D}\tr{L^{-1}_{33}}\right)(1-\psi_D).
\end{align*}
Operator-valued free probability theory then characterizes the limiting normalized block traces through a closed system of fixed-point equations, from which the limits of $\{E_i(z)\}_{i=0,\cdots,3}$ are recovered.
\subsubsection{Fixed-Point Equations }
We use the \texttt{auto-fpt} library of \citet{subramonian_auto-fpt_2025} to symbolically generate the fixed-point equations from the linear pencil in~\eqref{eqn:linear_pencil}. The library automates free-probability calculations for normalized traces of rational expressions in large random matrices. Applied to our pencil, it produces equations for the limiting block traces needed to recover $e_1,e_2,e_3$ and the Stieltjes transform of $M$. We present the resulting system below after eliminating redundant variables.

Let $\psi_p=p/d$, $\psi_n=n/d$, and $\psi_D=D/d$. We use zero-based block indices and write $G_{ij}=G_{ij}(z)$ for the scalar unknowns which are the appropriately normalized traces of square matrix blocks $(L^{-1})_{ij}$. Thus $G_{00}$ corresponds to the first diagonal block in the preceding display. The relations
\begin{equation}\label{eqn:fpe_substitutions}
    G_{16}=-\frac{z}{a}G_{17},\qquad
    G_{46}=-\frac{z}{a}G_{47},\qquad
    G_{28}=-\frac{hz}{a^4}G_{82}
\end{equation}
reduce the system to eighteen equations. Only the shared denominators are abbreviated:
\begingroup
\allowdisplaybreaks
\begin{align}
\mathcal D_0 &={}
\mu^3\psi_Da^2hzG_{14}
+\mu^2\psi_D\big[a^4(G_{44}-G_{74})+a^3zG_{64}-hzG_{11}\big]\notag\\
&\quad+\mu^2(1-\psi_D)\big[a^4G_{22}-hz(2G_{82}+G_{88})\big]\notag\\
&\quad+\mu\psi_D\big[a^2(G_{71}-G_{41})-azG_{61}\big]
+v_0^2(aG_{53}-zG_{33})-s^2z,\label{eqn:fpe_d0}\\
\mathcal D_1 &={}
\mu^2\psi_pG_{00}\Big[a^4\big((G_{33}-1)(G_{55}-1)-G_{35}G_{53}\big)\notag\\
&\qquad+a^3(zG_{35}+G_{53})-a^2zG_{33}+ahzG_{35}-hz\Big]\notag\\
&\quad+s^2z(aG_{35}-1),\label{eqn:fpe_d1}\\
\mathcal D_2 &={}
\mu\psi_D^2a^2s^2z(G_{14}G_{47}-G_{17}G_{44})
-\mu\psi_D\psi_na^2s^2zG_{14}\notag\\
&\quad+\mu\psi_D\psi_ns^2z^2G_{17}
-\psi_D\psi_ns^2zG_{47}
-\psi_D\psi_pa^2v_0^2G_{00}G_{44}\notag\\
&\quad+\psi_n^2s^2z+\psi_n\psi_pv_0^2zG_{00},\label{eqn:fpe_d2}\\
\mathcal D_3 &={}s^2z+\mu^2\psi_p(hz-a^4)G_{00}.\label{eqn:fpe_d3}
\end{align}
The fixed-point equations are as follows:
\begin{subequations}\label{eqn:fixed_point_system}
    \begin{align}
    G_{00} &=-\frac{s^2z}{\mathcal D_0},\\
G_{11}&=\frac{\mu^2\psi_pG_{00}\big[a^4(1-G_{55})+a^3zG_{35}\big]+s^2z(aG_{35}-1)}{\mathcal D_1},\label{eqn:fpe_g11}\\
G_{14}&=\frac{\mu\psi_paG_{00}\big[a(1-G_{55})+zG_{35}\big]}{\mathcal D_1},\label{eqn:fpe_g14}\\
G_{17}&=\frac{\mu\psi_pa^2G_{00}(aG_{35}-1)}{\mathcal D_1},\label{eqn:fpe_g17}\\
G_{41}&=-\frac{\mu a^2\big[\mu^2\psi_pG_{00}(-a^3G_{53}+a^2zG_{33}+hz)+s^2zG_{33}\big]}{\mathcal D_1},\label{eqn:fpe_g41}\\
G_{44}&=-\frac{\mu^2\psi_pG_{00}(-a^3G_{53}+a^2zG_{33}+hz)+s^2z}{\mathcal D_1},\label{eqn:fpe_g44}\\
G_{47}&=-\frac{\mu^2\psi_pa^4G_{00}(G_{33}-1)}{\mathcal D_1},\label{eqn:fpe_g47}\\
G_{61}&=-\frac{\mu a}{\mathcal D_1}\Big[\mu^2\psi_pG_{00}
\big(a^4(G_{33}G_{55}-G_{33}-G_{35}G_{53})+ahzG_{35}\big)\notag\\
&\hspace{30mm}+s^2zG_{33}\Big],\label{eqn:fpe_g61}\\
G_{64}&=-\frac{1}{\mathcal D_1}\Big[\mu^2\psi_pG_{00}
\big(a^3(G_{33}G_{55}-G_{33}-G_{35}G_{53})+hzG_{35}\big)\notag\\
&\hspace{30mm}+s^2zG_{35}\Big],\label{eqn:fpe_g64}\\
G_{71}&=-\frac{\mu a}{\mathcal D_1}\Big[\mu^2\psi_pG_{00}
\big(-a^4G_{53}+a^3z(G_{33}G_{55}-G_{35}G_{53})+ahzG_{55}\big)\notag\\
&\hspace{30mm}+s^2z\big(a(G_{33}G_{55}-G_{35}G_{53})+G_{53}\big)\Big],\label{eqn:fpe_g71}\\
G_{74}&=-\frac{1}{\mathcal D_1}\Big[\mu^2\psi_pG_{00}
\big(-a^3G_{53}+a^2z(G_{33}G_{55}-G_{35}G_{53})+hzG_{55}\big)\notag\\
&\hspace{30mm}+s^2zG_{55}\Big],\label{eqn:fpe_g74}\\
G_{33}&=\frac{\psi_ns^2z(\psi_n-\psi_DG_{47})}{\mathcal D_2},\label{eqn:fpe_g33}\\
G_{35}&=\frac{\psi_na\big(\mu\psi_Ds^2zG_{17}+\psi_pv_0^2G_{00}\big)}{\mathcal D_2},\label{eqn:fpe_g35}\\
G_{53}&=\frac{\psi_D\psi_ns^2z\left(aG_{44}-\frac{z}{a}G_{47}\right)}{\mathcal D_2},\label{eqn:fpe_g53}\\
G_{55}&=\frac{\psi_nz\big[\psi_ns^2+\psi_pv_0^2G_{00}-\mu\psi_Da^2s^2G_{14}+\mu\psi_Ds^2zG_{17}\big]}{\mathcal D_2},\label{eqn:fpe_g55}\\
G_{22}&=\frac{s^2z+\mu^2\psi_phzG_{00}}{\mathcal D_3},\label{eqn:fpe_g22}\\
G_{82}&=\frac{\mu^2\psi_pa^4G_{00}}{\mathcal D_3},\label{eqn:fpe_g82}\\
G_{88}&=\frac{s^2z-\mu^2\psi_pa^4G_{00}}{\mathcal D_3}.\label{eqn:fpe_g88}
\end{align}
\end{subequations}
\endgroup

\subsubsection{Recovering the Limiting Training and Test Errors}
We now combine~\eqref{eqn:fixed_point_system} with the spectral representations of test and train errors. 

\textbf{Step 1: Compute $\Gamma_{\psi_r}$}\\
To compute $\Gamma_{\psi_r}$, we need the limiting spectrum of $M$, for which we use the Stieltjes transform inversion formula. Limiting Stieltjes transform of $M$ is given by 
\begin{equation*}
    \textbf{s}_\nu(z) = \lim_{d\to\infty}\textbf{s}_{M}(z) = \frac{1}{z}\left(\left(1-G_{44}(z)\right)\psi_D+\left(1-G_{22}(z)\right)(1-\psi_D)-1\right).
\end{equation*}
Stieltjes inversion gives the upper-tail mass $F_M(x)$, rather than the cumulative mass below $x$:
\begin{equation*}
    F_{M}(x) = \frac{1}{\pi}\lim_{K\to\infty}\lim_{\eta\downarrow0}\Im\int_{x}^K\textbf{s}_{M}(u+i\eta)du.
\end{equation*}
Then, $\Gamma_{\psi_r}$ is obtained by solving
\begin{equation}\label{eq:app_cutoff_equation}
    F_{M}(\Gamma_{\psi_r}) = \psi_r.
\end{equation}

\textbf{Step 2: Use the limiting expressions for $\{E_i\}_{i=1,2,3}$}\\
Let $e_i = \lim_{d\to\infty} E_i$ for $i=1,2,3$. Translating the block-inverse identities into the zero-based indices of~\eqref{eqn:fixed_point_system} gives
\begin{align*}
    e_1(z) &= -\frac{\psi_p}{s^2z}G_{00}(z),\\
    e_2(z) &= -\frac{\psi_D}{a^2\mu}G_{17},\\
    e_3(z) &= \frac{1}{a^4\mu^2}\left(1-G_{22}\right)(1-\psi_D).
\end{align*}
Substituting these limits into the quadratic and cross terms of~\eqref{eq:app_test_expansion}, and using~\eqref{eq:app_stieltjes_trace} for the training error, yields the following expressions with the cutoff from~\eqref{eq:app_cutoff_equation}:
\begin{align*}
     \lim_{d\to\infty}\cE_{\text{test}}(\hat{A}_r)&= h\sigma^2+\left(1-2h\right)+\frac{h}{\pi}\lim_{K\to\infty}\lim_{\eta\downarrow0}\int_{\Gamma_{\psi_r}}^K dx\; x\Im \textbf{m}(x+i\eta),\\
     \lim_{d\to\infty}\cE_{\text{train}}(\hat{A}_r) &= h\sigma^2+\left(1-2h\right)-\frac{h}{\pi}\lim_{K\to\infty}\lim_{\eta\downarrow0}\int_{\Gamma_{\psi_r}}^K dx\; x\Im \textbf{s}_{\nu}(x+i\eta),
\end{align*}
where 
\begin{align*}
    \textbf{m}(z) = \frac{\mu_1^2(\sigma)}{\sigma^2}e_{2}(z)+\left(\frac{h\mu_1^2(\sigma)}{\sigma^2}-2a^4\frac{\mu_1^2(\sigma)}{\sigma^2}\frac{1}{z}\right)e_{3}(z)+{v^2(\sigma)}e_{1}(z).
\end{align*}
The transform $\textbf{m}$ combines the quadratic contribution with twice the cross contribution, recovering the error characterization of Theorem~\ref{thm:limiting_test_error}.

\subsubsection{Derivation of the Relevant Moments}\label{appndx:matrices}
We collect the Gaussian moment calculations used in the spectral reduction. The common tool is Mehler's expansion in Hermite polynomials: for a jointly standard Gaussian pair with law $P^\gamma$, correlations of nonlinear functions are expanded in powers of $\gamma$. In cross moments with a linear second factor, Hermite orthogonality selects the first coefficient. Feature second moments require separate diagonal and off-diagonal calculations.\\
\textbf{(i) We will first consider $\bar{V}_{\phi y}$:} We have
\begin{equation*}
    \bar{V}_{\phi y} = \bE{x,z}{\act{\rat{W}{d}(ax+\sqrt{h}z)}(ax+\sqrt{h}z)^T}\;.
\end{equation*}
Let $w_i$ denote the $i^{\text{th}}$ row of $W$ and let $x(j)$ denote the $j^{th}$ component of $x$. For large $d$, $\frac{\norm{w_i}^2}{d}$ concentrates to $1$. Note that, conditioned on $W$, $\frac{w_i^T(ax+\sqrt{h}z)}{\sqrt{d}}$ is a zero-mean Gaussian random variable with variance $\bE{x,z}{\frac{(w_i^T(ax+\sqrt{h}z))^2}{d}}=\frac{1}{d}w_i^T(a^2\cC+hI) w_i$. By the concentration property of quadratic forms, the variance concentrates to $\sigma^2 = \frac{1}{d}\tr{a^2\cC+hI}$ as $d\to\infty$. Let $\tilde{\cC} = a^2\cC+hI$. Let $\gamma_{ij} = \frac{1}{\sigma\sqrt{\tilde{\cC}_{jj}}}\rat{w_i^T\tilde{\cC}_{\cdot j}}{d}$
\begin{align*}
    [\bar{V}_{\phi y}]_{ij} &= \bE{x,z}{\act{\rat{w_i^T(ax+\sqrt{h}z)}{d}}(ax_j+\sqrt{h}z_j)},\\
    &= \sqrt{\tilde{\cC}_{jj}}\bE{(u,v)\sim P^{\gamma_{ij}}}{\act{\sigma u}v},\\
    &= \sqrt{\tilde{\cC}_{jj}}\sum_{k=0}^{\infty} \frac{(\gamma_{ij})^k}{k!}\bE{u}{\act{\sigma u}\text{He}_k(u)}\bE{v}{v\text{He}_k(v)},\\
    &= \sqrt{\tilde{\cC}_{jj}}\gamma_{ij}\bE{u}{\act{\sigma u}u} ,\\
    &= \frac{\mu_1(\sigma)}{\sigma}\rat{[W\tilde{\cC}]_{ij}}{d} \;.
\end{align*}
We have $\bar{V}_{\phi y} = \frac{\mu_1(\sigma)}{\sigma}\rat{W\tilde{\cC}}{d}$. \\

\textbf{(ii) Next, consider $\bar{V}_{\phi z}$:} \begin{align*}
    [\bar{V}_{\phi z}]_{ij} &= \bE{x,z}{\act{\rat{w_i^T(ax+\sqrt{h}z)}{d}}z_j},\\
    &= \bE{(x,u,v)\sim P^{\rat{w_{ij}}{d}}}{\act{\rat{a w_i^Tx}{d}+\sqrt{h}u}v},\\
    &= \sum_{k=0}^{\infty} \frac{(\rat{w_{ij}}{d})^k}{k!}\bE{x,u}{\act{\rat{a w_i^Tx}{d}+\sqrt{h}u}\text{He}_k(u)}\bE{v}{v\text{He}_k(v)},\\
    &= \rat{w_{ij}}{d}\bE{x,u}{\act{\rat{a w_i^Tx}{d}+\sqrt{h}u}u} ,\\
    &= \rat{w_{ij}}{d} \bE{g,u}{\varrho(a\sigma_0 g+\sqrt{h}u)u} ,\\
    &= \sqrt{h}\frac{\mu_1(\sigma)}{\sigma}\rat{w_{ij}}{d} \;.
\end{align*}
We have $\bar{V}_{\phi z} = \sqrt{h}\frac{\mu_1(\sigma)}{\sigma}\rat{W}{d}$. \\

\textbf{(iii)
Now we will consider $\bar{U}$:}
When considering $\bar{U}$, a similar argument allows to write $(\tilde{u}_0,\tilde{u}_1)=(\frac{1}{\sigma}\frac{w_i^T(ax+\sqrt{h}z)}{\sqrt{d}},\frac{1}{\sigma}\frac{w_j^T(ax+\sqrt{h}z)}{\sqrt{d}})\sim P^{\tilde{\gamma}_{ij}}$ with $\tilde{\gamma}_{ij} = \frac{1}{\sigma^2}\frac{w_i^T\tilde{\cC} w_j}{d}$, leading to 
\begin{align*}
    \bar{U}_{ij} &= \bE{x}{\act{\rat{w_i^Tx}{d}}\act{\rat{w_j^Tx}{d}}}\\
    &= \bE{(\tilde{u}_0,\tilde{u}_1)\sim P^{\tilde{\gamma}_{ij}}}{\varrho(\sigma\tilde{u}_0)\varrho(\sigma\tilde{u}_1)}\\
    &= \sum_{k=0}^{\infty} \frac{\tilde{\gamma}_{ij}^k}{k!} \;\bE{\tilde{u}_0}{\varrho(\sigma \tilde{u}_0)\text{He}_k(\tilde{u}_0)}\bE{\tilde{u}_1}{\varrho(\sigma \tilde{u}_1)\text{He}_k(\tilde{u}_1)}\\
    &= \sum_{k=0}^{\infty} \frac{\tilde{\gamma}_{ij}^k}{k!} \; \bE{\tilde{u}_0}{\varrho(\sigma\tilde{u}_0)\text{He}_k(\tilde{u}_0)}^2 \;.
\end{align*}
Separating diagonal and off-diagonal entries gives:
\begin{equation*}
    \bar{U}_{ij} = \begin{cases}
        \mu_0^2(\sigma)+\frac{\mu_1^2(\sigma)}{\sigma^2}\frac{w_i^T\tilde{\cC} w_j}{d} + O(1/d)\quad &\text{if } i\neq j \;,\\
        \norm{\varrho(\sigma\;\cdot)}^2\quad &\text{if } i=j \;.
    \end{cases}
\end{equation*}
The $\bigO{1/d}$ contribution cannot give rise to a $\bigO{1}$ change to the normalized trace quantities we are interested in. Hence, we neglect it. We have
\begin{equation*}
    \bar{U} = \mu_0^2(\sigma) \bone_p\bone_p^T + \frac{\mu_1^2(\sigma)}{\sigma^2}\rat{W}{d}\tilde{\cC}\rat{W^T}{d} + v^2(\sigma) I_p \;,
\end{equation*}
where $v^2(\kappa) = \norm{\varrho(\kappa\;\cdot)}^2-\mu_0^2(\kappa)-\mu_1^2(\kappa)$.\\

\textbf{(iv)
Now we will consider $V_{\phi z}$:}

Let
\begin{equation*}
    V_{\phi z}^l = \bE{z}{\act{\rat{W}{d}(ax_l+\sqrt{h}z)}z^T}\;.
\end{equation*}
Here $x_l$ is fixed and only the noise is averaged. Applying Mehler's formula as in the population cross-moment calculation gives
\begin{align*}
    [V_{\phi z}^l]_{ij} &= \bE{z}{\act{\rat{w_i^T(ax_l+\sqrt{h}z)}{d}}z_j},\\
    &= \bE{(u,v)\sim P^{\rat{w_{ij}}{d}}}{\act{\rat{a w_i^Tx_l}{d}+\sqrt{h}u}v},\\
    &= \sum_{k=0}^{\infty} \frac{(\rat{w_{ij}}{d})^k}{k!}\bE{u}{\act{\rat{a w_i^Tx_l}{d}+\sqrt{h}u}\text{He}_k(u)}\bE{v}{v\text{He}_k(v)},\\
    &= \rat{w_{ij}}{d}\bE{u}{\act{\rat{a w_i^Tx_l}{d}+\sqrt{h}u}u} + O(1/d),\\
    &= \rat{w_{ij}}{d}\varrho_1\left(\rat{w_i^Tx_l}{d}\right) + O(1/d)\;,
\end{align*}
where $\varrho_1(y) = \bE{u}{\varrho(ay+\sqrt{h}u)u}$. Summing over the $n$ data samples:
\begin{align*}
    [V_{\phi z}]_{ij} &= \frac{1}{n}\sum_{l=1}^{n} V^l_{ij},\\
    &= \rat{w_{ij}}{d} \frac{1}{n}\sum_{l=1}^{n} \varrho_1\left(\rat{ w_i^Tx_l}{d}\right),\\
    &= \rat{w_{ij}}{d} \bE{x}{\varrho_1\left(\rat{w_i^Tx}{d}\right)} + O(1/\sqrt{d}),\\
    &= \rat{w_{ij}}{d} \bE{g}{\varrho_1(\sigma_0g)} + O(1/\sqrt{d}),\\
    &= \rat{w_{ij}}{d} \bE{g,u}{\varrho(a\sigma_0g+\sqrt{h}u)u} + O(1/\sqrt{d}),\\
    &= \sqrt{h}\frac{\mu_1(\sigma)}{\sigma}\rat{w_{ij}}{d} + O(1/\sqrt{d})\;.
\end{align*}
Neglecting $O(1/\sqrt{d})$ terms, we have $V_{\phi z} = \sqrt{h}\frac{\mu_1(\sigma)}{\sigma}\rat{W}{d} = \bar{V}_{\phi z}$.

\textbf{(v)
Now we will consider $V_{\phi y}$:}

Let
\begin{align*}
    V_{\phi y}^l &= \bE{z}{\act{\rat{W}{d}(ax_l+\sqrt{h}z)}(ax_l+\sqrt{h}z)^T}\;,\\
    &= a\bE{z}{\act{\rat{W}{d}(ax_l+\sqrt{h}z)}}x_l^T+\sqrt{h}V_{\phi z}^l\;,\\
    &= a\varrho_0\left(\rat{W}{d}x_l\right)x_l^T+\sqrt{h}V_{\phi z}^l\;.
\end{align*}
Therefore, summing over the $n$ data samples:
\begin{align*}
    V_{\phi y} &= \frac{1}{n}\sum_{l=1}^{n} V^l_{\phi y}\;,\\
    &= a\frac{1}{n}\varrho_0\left(\rat{W}{d}X\right)X^T+\sqrt{h}V_{\phi z}\;.
\end{align*}

\textbf{(vi) Now, let's consider $U$:} Let
\begin{align*}
    U^l &= \bE{z}{\act{\rat{W}{d}(ax_l+\sqrt{h}z)}\act{\rat{W}{d}(ax_l+\sqrt{h}z)}^T}\;.
\end{align*}
For $i\neq j$, expand in the correlation between the two noise projections. The zeroth-order term gives the product of noise-averaged features, while the first-order term involves $\varrho_1$:
\begin{align*}
    U^l_{ij} &= \bE{z}{\act{\rat{w_i^T(ax_l+\sqrt{h}z)}{d}}\act{\rat{w_j^T(ax_l+\sqrt{h}z)}{d}}}\\
    &= \bE{(u,v)\sim P^{\frac{w_i^Tw_j}{d}}}{\act{a\rat{w_i^Tx_l}{d} + \sqrt{h}u}\act{a\rat{w_j^Tx_l}{d} + \sqrt{h}v}}\\
    &= \sum_{k=0}^{\infty} \frac{(\frac{w_i^Tw_j}{d})^k}{k!}\bE{u}{\act{a\rat{w_i^Tx_l}{d} + \sqrt{h}u}\text{He}_k(u)}\bE{v}{\act{a\rat{w_j^Tx_l}{d} + \sqrt{h}v}\text{He}_k(v)}\\
    &= \varrho_0\left(\rat{w_i^Tx_l}{d}\right)\varrho_0\left(\rat{w_j^Tx_l}{d}\right)+\frac{w_i^Tw_j}{d}\varrho_1\left(\rat{w_i^Tx_l}{d}\right)\varrho_1\left(\rat{w_j^Tx_l}{d}\right)+O(1/d)\;,
\end{align*}
where $\varrho_0(y) = \bE{u}{\varrho(ay+\sqrt{h}u)}$ and $\varrho_1(y) = \bE{u}{\varrho(ay+\sqrt{h}u)u}$. Summing over the $n$ data samples:
\begin{align*}
    U_{ij} &= \frac{1}{n}\sum_{l=1}^{n} U^l_{ij}\\
    &= \frac{1}{n}\sum_{l=1}^{n}\varrho_0\left(\rat{w_i^Tx_l}{d}\right)\varrho_0\left(\rat{w_j^Tx_l}{d}\right)+\frac{w_i^Tw_j}{d}\frac{1}{n}\sum_{l=1}^{n}\varrho_1\left(\rat{w_i^Tx_l}{d}\right)\varrho_1\left(\rat{w_j^Tx_l}{d}\right) + O(1/d)\\
    &= \frac{1}{n}\sum_{l=1}^{n}\varrho_0\left(\rat{w_i^Tx_l}{d}\right)\varrho_0\left(\rat{w_j^Tx_l}{d}\right)+\frac{w_i^Tw_j}{d}\bE{x}{\varrho_1\left(\rat{w_i^Tx}{d}\right)\varrho_1\left(\rat{w_j^Tx}{d}\right)} + O(1/d)\\
    &\stackrel{(a)}{=} \frac{1}{n}\sum_{l=1}^{n}\varrho_0\left(\rat{w_i^Tx_l}{d}\right)\varrho_0\left(\rat{w_j^Tx_l}{d}\right)+\frac{w_i^Tw_j}{d}\bE{g}{\varrho_1(\sigma_0g)}^2 + O(1/d)\\
    &= \frac{1}{n}\sum_{l=1}^{n}\varrho_0\left(\rat{w_i^Tx_l}{d}\right)\varrho_0\left(\rat{w_j^Tx_l}{d}\right)+h\frac{\mu_1^2(\sigma)}{\sigma^2}\frac{w_i^Tw_j}{d} + O(1/d)\;.\\
\end{align*}
In step $(a)$, we apply Mehler's formula to $\bE{x}{\varrho_1\left(\rat{w_i^Tx}{d}\right)\varrho_1\left(\rat{w_j^Tx}{d}\right)}$ and neglected terms that will lead to $O(1/d)$ terms.
The diagonal requires a separate calculation because a feature is perfectly correlated with itself. For $i=j$, we have:
\begin{align*}
    U^l_{ii} &= \bE{z}{\left(\act{\rat{w_i^T(ax_l+\sqrt{h}z)}{d}}\right)^2} \;,\\
\end{align*}
and
\begin{align*}
    U_{ii} &= \frac{1}{n}\sum_{l=1}^n\bE{z}{\left(\act{\rat{w_i^T(ax_l+\sqrt{h}z)}{d}}\right)^2}\\
    &= \bE{z,x}{\left(\act{\rat{w_i^T(ax+\sqrt{h}z)}{d}}\right)^2}+O(1/\sqrt{d})\\
    &= \norm{\act{\sigma\;\cdot}}^2+O(1/\sqrt{d}) \;.
\end{align*}
The $\bigO{1/\sqrt{d}}$ term in the above equation can be neglected, since there are only $\bigO{d}$ terms on the diagonal. Let $X=[x_1,x_2,\cdots,x_n]\in\R^{d\times n}$. We can write $U$ as:
\begin{equation*}
    U = \rat{\varrho_0\left(\rat{W}{d}X\right)}{n}\rat{\varrho_0\left(\rat{W}{d}X\right)^T}{n}+h\frac{\mu_1^2(\sigma)}{\sigma^2}\rat{W}{d}\rat{W^T}{d}+s^2I_p \;,
\end{equation*}
where
\begin{align*}
s^2 &= \norm{\varrho(\sigma\;\cdot)}^2-\bE{g}{\varrho_0(\sigma_0 g)^2}-\frac{h}{\sigma^2}\mu_1^2(\sigma)\\
&= \norm{\varrho(\sigma\;\cdot)}^2-\bE{g}{\bE{u}{\varrho(a\sigma_0 g+\sqrt{h}u)}^2}-\frac{h}{\sigma^2}\mu_1^2(\sigma)\\
&= \norm{\varrho(\sigma\;\cdot)}^2-c(a^2\sigma_0^2/\sigma^2,\sigma)-\frac{h}{\sigma^2}\mu_1^2(\sigma) \;,
\end{align*}
with $c(\gamma,\kappa) = \bE{u,v\sim P^\gamma}{\varrho(\kappa u)\varrho(\kappa v)}$. Let $H:= \varrho_0\left(\rat{W}{d}X\right)$.
We have found:
\begin{equation*}
    U = \rat{H}{n}\rat{H^T}{n}+h\frac{\mu_1^2(\sigma)}{\sigma^2}\rat{W}{d}\rat{W^T}{d}+s^2I_p \;,
\end{equation*}

\section{Additional Numerical Results}
\subsection{Test and train errors for different diffusion time $t$}\label{sec:add_res_time}

Figures~\ref{fig:sim_model_tanh_t0p1} and~\ref{fig:theory_model_tanh_t0p1} extend the learning curves to $t=0.1$. The main rank-dependent trends persist at this larger diffusion time: training error decreases with rank, while test error depends on both sample size and data dimension. In the displayed settings with $n>D$, test error reaches a minimum near $r=d-D$ and then increases, whereas for $n<D$ it increases throughout the displayed rank range. The random-feature predictions agree closely with the simulated training and test errors and reproduce these qualitative trends in the trained networks, supporting the relevance of the geometric reference rank beyond $t=0.01$.

\begin{figure}[h]
    \centering
    \includegraphics[width=\linewidth]{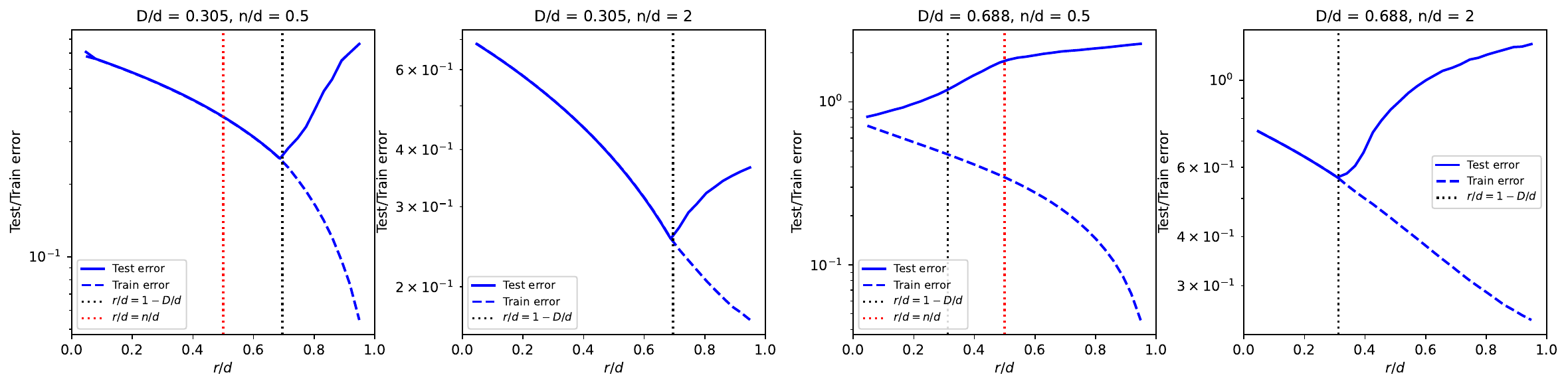}
    \caption{Simulation results for our numerical model, with $\tanh$ activation at $t=0.1$, $p/d=2$.}
    \label{fig:sim_model_tanh_t0p1}
\end{figure}

\begin{figure*}[h]
    \centering
    \includegraphics[width=\linewidth]{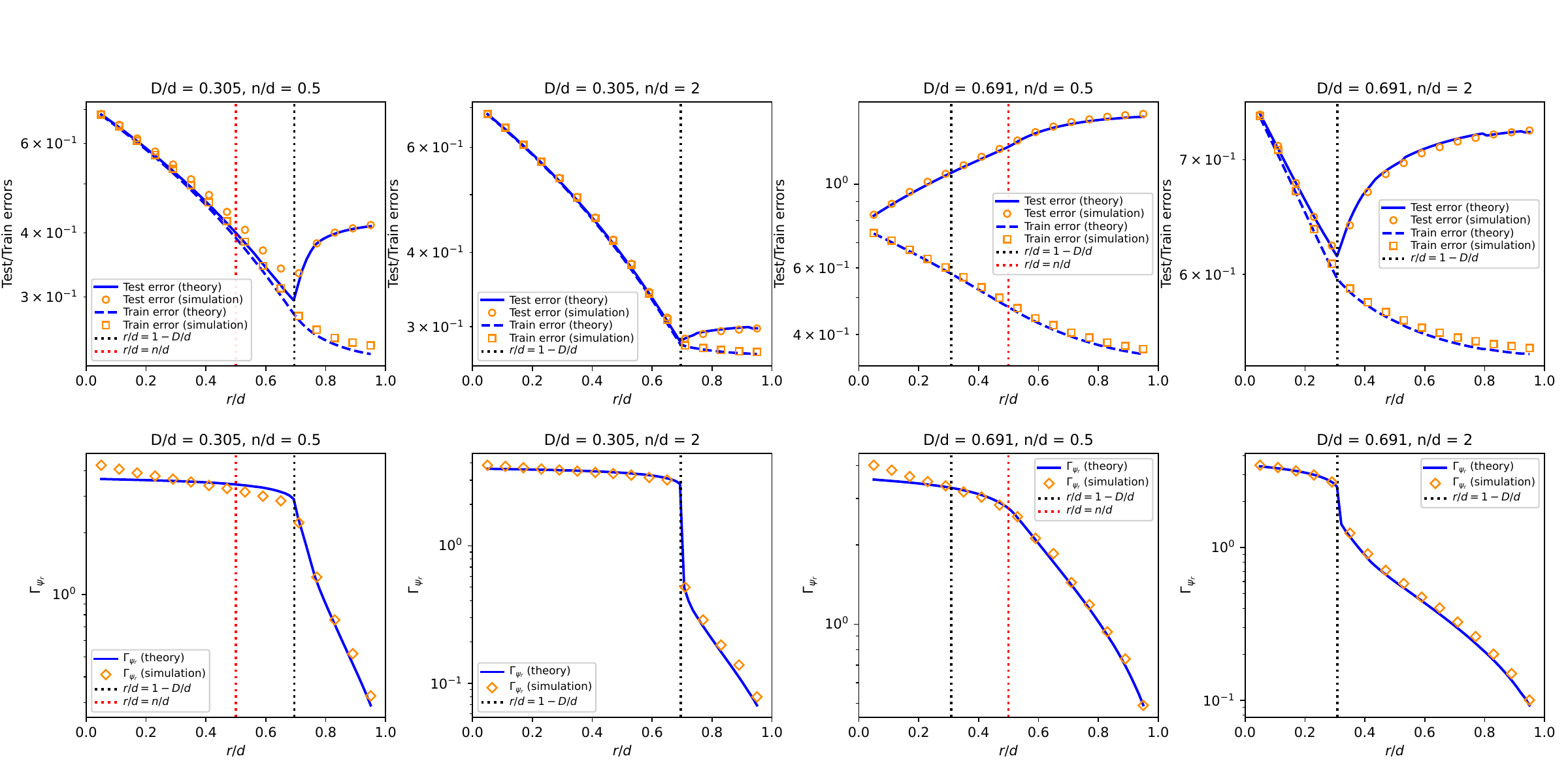}
    \caption{Analytical and simulation results for our theoretical model with $\tanh$ activation at $t=0.1$, $\psi_p = 2.0$.}
    \label{fig:theory_model_tanh_t0p1}
\end{figure*}

\section{Simulation Conditions}\label{app:simulation_conditions}
We describe the simulations of the trainable score network in~\eqref{eqn:score_numerical}, reported in Figures~\ref{fig:sim_model_tanh_t0p01} and~\ref{fig:sim_model_tanh_t0p1}.

\paragraph{Data and parameter sweep.}
We set the ambient dimension to $d=256$ and the hidden-layer width to $p=512$. For each dataset seed, an orthonormal basis $\Theta\in\mathbb{R}^{d\times D}$ is obtained by a QR decomposition of a standard Gaussian matrix. Independent training and test samples are drawn as $x=\sqrt{D}\,\Theta g/\|g\|$, with $g\sim\mathcal{N}(0,I_D)$. We use $D\in\{78,176\}$ and $n\in\{128,512\}$, corresponding to $D/d\approx0.305,0.688$ and $n/d=0.5,2$. Each test set contains $512$ independent clean samples from the same distribution. 

\paragraph{Initialization and training.}
The activation is $\tanh$, and the factors $W_t,A_t,B_t$ are initialized with independent standard Gaussian entries, with the normalization factors specified in~\eqref{eqn:score_numerical}. Biases are initialized to zero.
All three weight matrices and the biases are trained for $30{,}000$ Adam \citep{kingma_adam_2017} updates with a constant learning rate of $10^{-2}$, batch size $512$, no weight decay, and global gradient-norm clipping at $10$. Each update samples clean training points uniformly with replacement and draws fresh standard Gaussian noise, forming $y=a_tx+\sqrt{h_t}z$. The optimized loss is the mean of $\|\sqrt{h_t}s_t(y)+z\|^2/d$ over the minibatch. Within each seed and diffusion time, models at different ranks share initial weight entries, minibatch indices, and noise realizations, but maintain separate trainable parameters and optimizer states. The rank constraint is implemented by retaining the first $r$ terms of the factorized output map.

\paragraph{Evaluation and averaging.}
Training and test errors evaluate the same normalized DSM objective on all clean training points and on the independent test set, respectively. Each evaluation averages over $16$ fresh noise realizations per clean sample, processed in batches of $256$; evaluation noise is separate from training noise and shared across ranks at a fixed diffusion time. 
Plotted curves use the final evaluation at $30{,}000$ updates. We average over three paired dataset and initialization seeds, $(0,0)$, $(1,1)$, and $(2,2)$.
The implementation uses JAX, Flax, and Optax with double-precision arrays.

\end{document}